\documentclass[11pt,reqno,twoside]{article}

\usepackage{fixltx2e} %

\usepackage{cmap} %

\usepackage[T1]{fontenc}
\usepackage[utf8]{inputenc}
\usepackage{graphicx}
\usepackage{placeins}
\usepackage{enumerate}

\usepackage{subcaption}

\usepackage{verbatim}

\usepackage{setspace}

\let\counterwithin\relax  %
\usepackage{lmodern} %
\usepackage[scale=0.88]{tgheros} %

\usepackage{bm} %

\usepackage{amsmath,amsbsy,amsgen,amscd,amsthm,amsfonts,amssymb} 

\usepackage[centering,top=1.5in,bottom=1.2in,left=1in,right=1in]{geometry}

\usepackage{titling}
\usepackage{cases}

\graphicspath{ {./images/} }

\usepackage[sf,bf,compact]{titlesec}

\usepackage{booktabs,longtable,tabu} %
\usepackage[font=small,margin=12pt,labelfont={sf,bf},labelsep={space}]{caption}

\usepackage[usenames,dvipsnames]{xcolor}

\definecolor{dark-gray}{gray}{0.3}
\definecolor{dkgray}{rgb}{.4,.4,.4}
\definecolor{dkblue}{rgb}{0,0,.5}
\definecolor{medblue}{rgb}{0,0,.75}
\definecolor{rust}{rgb}{0.5,0.1,0.1}

\usepackage{url}
\usepackage[colorlinks=true]{hyperref}
\hypersetup{linkcolor=dkblue}    
\hypersetup{citecolor=rust}      
\hypersetup{urlcolor=rust}     

\usepackage[final]{microtype}

\newtheoremstyle{myThm} %
    {\topsep}                    %
    {\topsep}                    %
    {\itshape}                   %
    {}                           %
    {\sffamily\bfseries}                   %
    {.}                          %
    {.5em}                       %
    {}  %

\newtheoremstyle{myRem} %
    {\topsep}                    %
    {\topsep}                    %
    {}                   %
    {}                           %
    {\sffamily}                   %
    {.}                          %
    {.5em}                       %
    {}  %

\newtheoremstyle{myDef} %
    {\topsep}                    %
    {\topsep}                    %
    {}                   %
    {}                           %
    {\sffamily\bfseries}                   %
    {.}                          %
    {.5em}                       %
    {}  %

\theoremstyle{myThm}
\newtheorem{theorem}{Theorem}[section]
\newtheorem{lemma}[theorem]{Lemma}
\newtheorem{proposition}[theorem]{Proposition}
\newtheorem{corollary}[theorem]{Corollary}

\newtheorem{assumption}[theorem]{Assumption}

\theoremstyle{myRem}
\newtheorem{remark}[theorem]{Remark}

\theoremstyle{myDef}

\let\originalleft\left
\let\originalright\right
\renewcommand{\left}{\mathopen{}\mathclose\bgroup\originalleft}
\renewcommand{\right}{\aftergroup\egroup\originalright}

\usepackage{mathtools}
\mathtoolsset{centercolon}  %

\definecolor{mygreen}{rgb}{0.1,0.75,0.2}

\usepackage[]{algorithm}
\usepackage{algorithmic}

\usepackage{graphicx}

\usepackage{soul}
\usepackage{authblk}
\usepackage[square,numbers]{natbib}
\usepackage{chngcntr}
\usepackage{mathrsfs}
\counterwithin{table}{section}
\counterwithin{algorithm}{section}

\title{Optimization on Manifolds via Graph Gaussian Processes}   %
\author{Hwanwoo Kim$^1$, Daniel Sanz-Alonso$^1$, and Ruiyi Yang$^2$}
\vspace{.25in} 

\date{$^1$University of Chicago, $^2$Princeton University}

\makeatletter\@addtoreset{section}{part}\makeatother%
\numberwithin{equation}{section}

\newcommand{\upperRomannumeral}[1]{\uppercase\expandafter{\romannumeral#1}}

\renewcommand{\hat}{\widehat}

\graphicspath{ {./figures/} }

\newcommand{\R}{\mathbb{R}}

\newcommand{\M}{\mathcal{M}}
\newcommand{\Nc}{\mathcal{N}}

\newcommand{\Ma}{\text{\tiny Ma}}
\newcommand{\SE}{\text{\tiny SE}}

\newcommand{\MNmax}{z_{\mathcal{M}_N}^*}
\newcommand{\Mmax}{z_{\mathcal{M}}^*}

\newcommand{\flevy}{f^{\text{\tiny Levy}}}
\newcommand{\fackley}{f^{\text{\tiny Ackley}}}
\newcommand{\frastri}{f^{\text{\tiny Rastrigin}}}
\usepackage{mathtools}

\begin{document}

\maketitle

\begin{abstract}
This paper integrates manifold learning techniques within a \emph{Gaussian process upper confidence bound} algorithm to optimize an objective function on a manifold. Our approach is motivated by applications where a full representation of the manifold is not available and querying the objective is expensive. We rely on a point cloud of manifold samples to define a graph Gaussian process surrogate model for the objective. Query points are sequentially chosen using the posterior distribution of the surrogate model given all previous queries. We establish regret bounds in terms of the number of queries and the size of the point cloud.
Several numerical examples complement the theory and illustrate the performance of our method.
\end{abstract}

\section{Introduction}\label{sec:introduction}
 Optimization problems on manifolds are ubiquitous in science and engineering. For instance, low-rank matrix completion and rotational alignment of 3D bodies can be formulated as optimization problems over spaces of matrices that are naturally endowed with manifold structures.
These matrix manifolds belong to agreeable families \cite{ye2022optimization} for which Riemannian gradients, geodesics, and other geometric quantities have closed-form expressions that facilitate the use of Riemannian optimization algorithms \cite{edelman1998geometry,absil2009optimization,boumal2020introduction}. 
In contrast, this paper is motivated by optimization problems where the search space is a manifold that  the practitioner can only access through a discrete point cloud representation, preventing direct use of Riemannian optimization algorithms. Moreover, the hidden manifold may not belong to an agreeable family, further hindering the use of classical methods. 
Illustrative examples where manifolds are represented by point cloud data include computer vision, robotics, and shape analysis of geometric morphometrics \cite{hein2005intrinsic,gao2019gaussian,ruiyilocalregularization}. Additionally, across many applications in data science, high-dimensional point cloud data contains low-dimensional structure that can be modeled as a manifold for algorithmic design and theoretical analysis \cite{coifman2006diffusion,belkin2006manifold,trillos2020consistency}. Motivated by these problems,  this paper introduces a Bayesian optimization method with convergence guarantees to optimize an expensive-to-evaluate function on a point cloud of manifold samples.

To formalize our setting, consider the optimization problem
\begin{equation}\label{eq:manifoldoptimization}
    \text{maximize}\,\, f(x), \quad \quad  x \in \mathcal{M}_N, 
\end{equation}
where $\mathcal{M}_N=\{x_i\}_{i=1}^N$ is a collection of samples from a   compact   manifold $\mathcal{M}\subset \mathbb{R}^d$. We assume that the manifold $\M$ is unknown to the practitioner, but that they have access to the samples $\M_N.$
The objective function $f$ in \eqref{eq:manifoldoptimization} is defined on the hidden manifold $\mathcal{M}$; however, since $\M$ is unknown, we restrict the search domain to the given point cloud $\mathcal{M}_N.$  Motivating examples include locating the portion of busiest traffic along a highway (idealized as a one-dimensional manifold), or finding the point of highest temperature on an artificial surface for material design. In these and other applications, the search domains are manifolds for which only a discrete representation may be available. 
As a result, Riemannian optimization methods \cite{edelman1998geometry,absil2009optimization,boumal2020introduction,hu2020brief,ye2022optimization} that require Riemannian gradients or geodesics are not directly applicable.

While being discrete, the optimization problem \eqref{eq:manifoldoptimization} is challenging when the objective function $f$ is expensive to evaluate due to computational, monetary, or opportunity costs. For instance, querying $f$ may involve numerically solving a system of partial differential equations, placing a sensor at a new location, or time-consuming human labor. In such cases, solving \eqref{eq:manifoldoptimization} by exhaustive search over $\M_N$ is unfeasible for large $N,$ and it is important to design optimization algorithms that provably require fewer evaluations of the objective than the size $N$ of the point cloud.
Solving \eqref{eq:manifoldoptimization} is also challenging in applications where the objective function does not satisfy structural assumptions (e.g. concavity or linearity) other than a sufficient degree of smoothness, and in applications where $f$ is   a \emph{black-box} in that one has only access to noisy output from $f$ rather than to an analytic expression of this function. We refer to \cite{frazier2018tutorial} for a survey of problems where these conditions arise.

Motivated by these geometric and computational challenges, we introduce an approach to solve \eqref{eq:manifoldoptimization} that works directly on the point cloud $\mathcal{M}_N$ and necessitates few evaluations of the objective. In particular, we show that in the large $N$ limit and under suitable smoothness assumptions, our method provably requires far fewer evaluations of the objective than the size $N$ of the point cloud. Our algorithm falls in the general framework of Bayesian optimization and is specifically designed to achieve such a convergence guarantee. The main focus will be on the mathematical analysis of the proposed approach, but we also present simulation studies to illustrate and complement our theory.

\subsection{Overview of our Approach}
The problem features that gradients are not available and evaluation of the objective is expensive naturally lead us to adopt a Bayesian optimization (BO) approach to solve \eqref{eq:manifoldoptimization}.  
BO is an iterative procedure that relies on solving a sequence of surrogate optimization problems to sidestep the need of gradient information on $f.$ 
At each iteration, the surrogate problem is to optimize an \emph{acquisition function} defined using  a \emph{probabilistic model} of the objective function conditioned to previous iterates. The acquisition function should be inexpensive to evaluate and optimize, and at the same time provide useful information about where the optimizer of $f$ is most likely to lie. 
The probabilistic model should be sufficiently rich to adequately represent the objective function. Many choices of acquisition function have been proposed in the literature, including expected improvement, entropy search, and knowledge gradient (see \cite{frazier2018tutorial} for a review). 
Popular probabilistic models for $f$ include 
 Gaussian processes \cite{williams2006gaussian,gramacy2020surrogates} and  Bayesian additive regression trees \cite{chipman2010bart}. 
Adequately choosing the acquisition function and the probabilistic model is essential to the success of BO algorithms.

The BO method that we propose and analyze has the distinctive feature that both the probabilistic model and the acquisition function are carefully chosen to ensure convergence of the returned solution to a global maximizer of $f$ under suitable smoothness assumptions. 
A natural way to characterize the smoothness of $f$ is to assume it is a sample path from a Gaussian process (GP) defined on $\mathcal{M}$. Under this smoothness assumption, we adopt a graph GP model \cite{sanz2020spde,borovitskiy2021matern} for $f|_{\mathcal{M}_N}$, the restriction of $f$ to the point cloud. The graph GP is designed to be a discretely indexed GP that approximates a Mat\'ern or squared exponential GP on the hidden manifold $\M$ as the size of the point cloud grows to infinity. Applications of graph GPs in Bayesian inverse problems, spatial statistics, and semi-supervised learning are discussed in \cite{sanz2020spde,trillos2022mathematical,harlim2020kernel,josh2021}. In this paper, we extend the convergence analysis for Mat\'ern graph GP models in \cite{sanz2020spde,sanz2020unlabeled,trillos2020consistency,garcia2018continuum} to also cover squared exponential kernels, see Proposition \ref{prop:graph GP approx bound}.

Such error analysis is important since it allows us to quantify the misspecification error when modeling $f|_{\mathcal{M}_N}$ with a graph GP. In particular, the model that we use for computation does not necessarily match the true distribution of $f|_{\mathcal{M}_N}$ due to the missing information about $\mathcal{M}$; to obtain convergence guarantees, this geometric misspecification needs to be corrected by suitably choosing the acquisition function. We accomplish this goal by applying the framework developed in \cite{bogunovic2021misspecified}. In so doing, we adapt their formulation to cover our problem setting, where $f$ is a sample path from a GP instead of an element of a reproducing kernel Hilbert space.

\subsection{Contributions and Related Work}
Our careful choice of probabilistic model and acquisition function allows us to establish a bound on the simple regret (see \eqref{eq:simple regret} for its definition) that converges to zero as the number $L$ of evaluations of the objective and the size $N$ of the point cloud converge to infinity while keeping the relation $L\ll N$ (see Theorem \ref{thm:regret bound},  Remark \ref{remark:L<<N},  and Corollary \ref{cor:continuum regret}). In other words, our algorithm can provably find a global maximizer of $f$ as we acquire more samples from the   compact   manifold   $\M$   while still keeping the number of evaluations of the objective much smaller than the size of the point cloud. We are not aware of an existing algorithm to solve \eqref{eq:manifoldoptimization} that enjoys a similar convergence guarantee. Synthetic computed examples will complement the theory, illustrate the applicability of our method, and showcase the importance of incorporating geometric information in the probabilistic model. 

As noted in \cite{frazier2018tutorial}, BO algorithms have been most popular in continuous Euclidean domains. Methods that are tailored to manifold settings \cite{jaquier2020bayesian,jaquier2022geometry} and discrete spaces  \cite{baptista2018bayesian,luong2019bayesian,swersky2020amortized,deshwal2021bayesian} have received less attention.
On the one hand, 
the search domain in our setting \eqref{eq:manifoldoptimization} is a discrete subset of a manifold, 
 but naive application of discrete BO (e.g. using a standard Euclidean GP on the ambient space $\mathbb{R}^d$) would fail to adequately exploit the geometric information contained in the point cloud; in particular, it would fail to suitably encode smoothness of the probabilistic model for $f$ along the hidden manifold $\M.$ The empirical advantage of our graph GPs over Euclidean kernels will be illustrated in our numerical experiments (see Subsection \ref{sec:ex-cow}). On the other hand, 
 the manifold in our setting is only available as a point cloud, which precludes the use of manifold BO approaches \cite{jaquier2020bayesian,jaquier2022geometry} that require access to geodesic distances and eigenpairs of the Laplace-Beltrami operator on $\mathcal{M}$ for modeling $f$, and to Riemannian gradients for optimizing the acquisition function. 
 Therefore, our algorithm solves a practical problem for which limited tools with theoretical guarantees are available. 
 In the context of Riemannian optimization, our algorithm is still applicable when the differential geometric quantities necessary for gradient-based methods are not readily available.
 A closely related work in this direction is \cite{shustin2022manifold}, which also assumes a point cloud representation of the manifold but instead reconstructs from it tangent spaces, gradients, and retractions, followed by an approximate Riemannian gradient descent. Our paper differs from \cite{shustin2022manifold} in that our algorithm is based on Bayesian optimization and no gradient approximation is carried out, as a result of which we do not need to assume the point cloud to be quasi-uniform. 
 Going beyond manifold constraints, optimization of functions with low effective dimensionality has been addressed in \cite{wang2016bayesian,kirschner2019adaptive,cartis2023bound,cartis2023global} employing subspace methods (see also the references therein).

\subsection{Outline}
\begin{itemize}
    \item Section \ref{sec:problemsetting} introduces the  \emph{graph Gaussian process upper confidence bound} (GGP-UCB)  algorithm and describes the choice of surrogate model and acquisition function. Our main result, Theorem \ref{thm:regret bound}, establishes convergence rates. 
    \item Section \ref{sec:implementation} discusses important practical considerations such as estimating the parameters of the surrogate model and tuning the acquisition function. 
    \item Section \ref{sec:simulation} contains numerical examples that illustrate and complement the theory.
    \item Section \ref{sec: conclusion} closes with a summary of our paper and directions for further research.
    \item The proofs of our main results can be found in the appendices.
\end{itemize}

\subsection{Notation}
For $a,b$ two real numbers, we denote $a\wedge b=$ min$\{a,b\}$ and $a\vee b=$ max$\{a,b\}$. The symbol $\lesssim$ will denote less than or equal to up to a universal constant. For two real sequences $\{a_i\}$ and $\{b_i\}$, we denote   (i) $a_i\ll b_i$ if $\operatorname{lim}_i (a_i/b_i)=0$;    (ii) $a_i=O(b_i)$ if $\operatorname{lim\, sup}_i (a_i/b_i)\leq C$ for some positive constant $C$; and  (iii) $a_i\asymp b_i$ if  $c_1\leq \operatorname{lim\,inf}_i (a_i/b_i) \leq \operatorname{lim\,sup}_i (a_i/b_i) \leq c_2$ for some positive constants $c_1,c_2$.

\section{The GGP-UCB Algorithm}\label{sec:problemsetting}
In this section we introduce our algorithm and establish convergence guarantees. We start in 
Subsection \ref{ssec:problemformulation} by formalizing the problem setting. 
Subsection \ref{ssec:mainalgorithm} describes the main GGP-UCB algorithm. The choice of surrogate model and acquisition function are discussed in
Subsections \ref{ssec:prior} and \ref{ssec:acquisition}, respectively. Finally, Subsection \ref{ssec:regret bound} presents our main theoretical result, Theorem \ref{thm:regret bound}.

\subsection{Problem Formulation}\label{ssec:problemformulation}
Let $f$ be a function defined over a compact Riemannian submanifold $\mathcal{M}\subset \mathbb{R}^d$ of dimension $m$. Suppose that a full representation of $\mathcal{M}$ is not available and we are only given the dimension $m$ and a point cloud of manifold samples $\{x_i\}_{i=1}^N =: \M_N \subset \mathcal{M}$.  We are interested in solving the optimization problem
\begin{align}\label{eq:discrete opt}
    \underset{x\in\mathcal{M}_N}{\operatorname{max}}\,\,f(x)
\end{align}
in applications where the objective $f$ is expensive to evaluate and we may only collect $L \ll N$ noisy measurements $y_\ell$ of the form
\begin{align} \label{eq:observations}
    y_{\ell}= f(z_{\ell})+ \eta_{\ell}, \qquad \eta_{\ell}\overset{i.i.d.}{\sim} \mathcal{N}(0,\sigma^2),\quad 1\leq \ell \leq L, 
\end{align}
where $\{z_\ell\}_{\ell = 1}^L$ are query points and $\sigma$ is a given noise level. 
The goal is then to solve \eqref{eq:discrete opt} with $L\ll N$ queries of $f$.

Let  $\mathcal{Z}_L := \{z_\ell\}_{\ell =1}^L \subset \M_N$ denote the query points sequentially found by our algorithm, introduced in Subsection \ref{ssec:mainalgorithm} below. We shall quantify the performance of our approach using the \emph{simple regret}, defined as 
\begin{align}
    r_{N,L}:=f(\MNmax)-f(z_L^*), \quad\quad  \MNmax=\underset{z\in \M_N}{\operatorname{arg\,max}}\, f(z) ,\quad z_L^*= \underset{z\in \mathcal{Z}_L}{\operatorname{arg\,max}}\, f(z). \label{eq:simple regret}
\end{align}
Note that the simple regret depends both on the number $L$ of queries and on the size $N$ of the point cloud, since $\MNmax$ and $z_L^*$ both depend implicitly on $N$. One should interpret $N$ as a large fixed number and $L$ as the running index. The dependence on $N$ of the query points $z_{\ell}$'s will be omitted for notational simplicity.

\begin{remark}\label{remark:opt over discrete}
The optimizer $\MNmax$ over the point cloud $\M_N$ is not necessarily the global optimizer of $f$ over $\mathcal{M}$. Since we only have access to $\M_N$, finding the maximizer over $\M_N$ is the best we can hope for without  reconstructing or estimating the hidden manifold $\M$. 
 Nevertheless, we will show in Corollary \ref{cor:continuum regret} that the \emph{continuum regret}, defined as 
\begin{align}\label{eq:continuum regret}
    r_{N,L}^{\tiny \operatorname{cont}}:=f(\Mmax)-f(z_L^*), \qquad \Mmax=\underset{z\in\mathcal{M}}{\operatorname{arg\,max}}\,\, f(z), \quad z_L^*=\underset{z\in\mathcal{Z}_L}{\operatorname{arg\,max}}\,\, f(z), 
\end{align}
also converges to zero as both $N$ and $L$ approach infinity while keeping $L\ll N$ if the $x_i$'s satisfy Assumption \ref{assp:manifold}. In other words, the maximizer $z_L^*$ returned by our algorithm is an approximate global maximizer of $f$ over $\mathcal{M}$ despite the fact that $z_L^*\in \mathcal{M}_N$. 
$\hfill \square$
\end{remark}

\subsection{Main Algorithm}\label{ssec:mainalgorithm}
The Bayesian approach to optimization starts by constructing a GP model for the function to be optimized. We recall that a GP  with mean $\mu(\cdot)$ and covariance $c(\cdot,\cdot)$ is a stochastic process where the joint distribution over any finite set of indices $s_1,\ldots,s_n$ is a multivariate Gaussian with mean vector $[\mu(s_i)]_{i=1}^n$ and covariance matrix $[c(s_i,s_j)]_{i,j=1}^n$ \cite{williams2006gaussian}. The mean and covariance functions together encode information about the values of the function, their correlation, and their uncertainty.

In our setting, we need to construct a GP surrogate prior model $\pi_N$ for $f_N,$ where $\pi_N$ would simply be an $N$-dimensional multivariate Gaussian. A natural requirement is that, for $u_N \sim \pi_N,$ $u_N(x_i)$ and $u_N(x_j)$ should be highly correlated iff $x_i$ and $x_j$ are close \emph{along the manifold}, that is, if the geodesic distance $d_\M(x_i,x_j)$ is small. We shall discuss in Subsection \ref{ssec:prior} prior models $\pi_N$ that fulfill this requirement. Defining the covariance matrix of  $\pi_N$ by using a standard covariance function in the Euclidean space $\R^d$ would in general fail to meet this requirement, since two points may be close in Euclidean space but far apart in terms of the geodesic distance $d_\M$ in $\M.$ %

Once a choice of surrogate prior model is made, the next step is to sequentially find query points by maximizing an acquisition function \cite{srinivas2010gaussian}.  Suppose we have picked query points $z_1,\ldots,z_{\ell-1}$ in the first $\ell-1$ iterations and obtained noisy measurements \begin{align}
    y_k = f(z_k) + \eta_k, \quad \quad \eta_k\overset{i.i.d.}{\sim} \mathcal{N}(0,\sigma^2),\quad \quad  1 \le k  \le \ell-1. \label{eq:obs model}
\end{align}
At the $\ell$-th iteration, we will  pick the next query point $z_{\ell}$ by maximizing  an upper confidence bound acquisition function \cite{srinivas2010gaussian,bogunovic2021misspecified} of the form
\begin{equation}\label{eq:acquisition}
     A_{N,\ell}(z)=\mu_{N,\ell-1}(z)+B_{N,\ell} \sigma_{N,\ell-1}(z), \qquad z\in\mathcal{M}_N,
\end{equation}
where $B_{N,\ell}$ is a user-chosen parameter, and $\mu_{N,\ell-1},$  $\sigma_{N,\ell-1}$ are the mean and standard deviation of the posterior distribution $\pi_N(\cdot\,|\,y_1,\ldots,y_{\ell-1}).$ Denoting by $c_N(\cdot,\cdot)$ the covariance function of the surrogate prior $\pi_N$, i.e., $c_N(x_i,x_j)$ is the covariance between $u_N(x_i)$ and $u_N(x_j)$ for $u_N \sim \pi_N$, we have the expressions 
\begin{equation}\label{eq:pm pstd}
    \begin{aligned}
    \mu_{N,\ell-1}(z)&=c_{N,\ell-1}(z)^{\top}(C_{N,\ell-1}+\sigma^2 I)^{-1}Y_{\ell-1},\\
    \sigma^2_{N,\ell-1}(z)&=c_N(z,z)-c_{N,\ell-1}(z)^{\top}(C_{N,\ell-1}+\sigma^2 I)^{-1}c_{N,\ell -1}(z), 
\end{aligned}\qquad  z\in\mathcal{M}_N,
\end{equation}
where $Y_{\ell-1}=(y_1,\ldots,y_{\ell-1})^{\top}\in\mathbb{R}^{\ell-1}$, $c_{N,\ell-1}(z) \in\mathbb{R}^{\ell-1}$ is a vector with entries $\bigl(c_{N,\ell-1}(z)\bigr)_i=c_N(z,z_i)$, and $C_{N,\ell-1}\in\mathbb{R}^{\ell-1\times \ell-1}$ is a matrix with entries $(C_{N,\ell-1})_{ij}=c_{N}(z_i,z_j)$.

\begin{algorithm}
\caption{The GGP-UCB Algorithm} \label{algo:GP-UCB}
\begin{algorithmic}
\REQUIRE Point cloud $\M_N$; prior $\pi_N$; initialization $z_0$; total iterations $L$; parameters $\{B_{N,\ell}\}_{\ell=1}^L$.  
\FOR{$\ell=1,\ldots,L$}
\STATE Observe $y_{\ell-1}=f(z_{\ell-1})+\eta_{\ell-1}$,  with $\eta_{\ell-1}\overset{i.i.d.}{\sim} \mathcal{N}(0,\sigma^2).$ 
\STATE Compute $\mu_{N,\ell-1}$ and $\sigma_{N,\ell-1}$ based on $\{(z_k,y_k)\}_{k=0}^{\ell-1}$.
\STATE Choose $z_{\ell}=\underset{z\in \M_N}{\operatorname{arg\,max}}\,\, \Bigl\{\mu_{N,\ell-1}(z)+B_{N,\ell}\sigma_{N,\ell-1}(z) \Bigr\}$.
\vspace{-.15cm} \ENDFOR
\ENSURE $z_1,\ldots,z_L$.
\end{algorithmic}
\end{algorithm}

The GGP-UCB method is summarized in Algorithm \ref{algo:GP-UCB}. The intuition is that maximizing the acquisition function \eqref{eq:acquisition} represents a compromise between choosing points where the mean of the surrogate is large (exploitation) and where the variance is large (exploration). The parameter $B_{N,\ell}$ balances these two competing goals and its choice is crucial to the performance of the algorithm. In particular, we will discuss in Subsection \ref{ssec:acquisition} a choice of $B_{N,\ell}$ that helps correct for misspecification arising from the point cloud representation of $\M,$ and we will discuss in Subsection \ref{ssec:tuningB} a practical approach for tuning $B_{N,\ell}$ empirically. Finally, we point out that in practice one may choose to return as output of the algorithm the candidate $z_{\ell}$ that leads to the largest observation $y_{\ell}$ when the noise is small, or, otherwise, the $z_{\ell}$ that maximizes the posterior mean at the $L$-th iteration, i.e., the mean $\mu_{N,L}$ of $\pi_N(\cdot\,|\, y_1,\ldots,y_L)$.

\subsection{Choice of Prior: Graph Gaussian Processes (GGPs)} \label{ssec:prior}
  In this subsection we review the construction of GGP models for $f_N$, the restriction of $f$ to the $x_i$'s.    We first give a brief overview of manifold GPs before describing GGPs.  Manifold GPs will be used in our theoretical analysis, but are not implementable in our setting since the manifold $\M$ is unknown to the practitioner.   The presentation in this subsection follows  \cite{sanz2020spde,borovitskiy2020matern} and readers familiar with manifold GPs and GGPs can skip to Proposition \ref{prop:graph GP approx bound}.  
\subsubsection{Manifold GP Models} 
Since $f$ is a function over $\mathcal{M},$ it will be useful to start by recalling the construction of GPs over $\mathcal{M}$. A naive approach would be to simply use geodesic distances instead of Euclidean ones in covariance functions such as the Mat\'ern and squared exponential (SE) 
\begin{align}
    c_{\nu,\kappa}^{\Ma}(x,\tilde{x}) = \frac{2^{1-\nu}}{\Gamma(\nu)} \left(\kappa|x-\tilde{x}|\right)^{\nu}K_{\nu}\left(\kappa|x-\tilde{x}|\right),\qquad  c^{\SE}_{\tau}(x,\tilde{x})= \exp\left(-\frac{|x-\tilde{x}|^2}{4\tau}\right), \label{eq:Ma SE cf}
\end{align}
where $|\cdot|$ denotes the Euclidean distance, $\Gamma$ is the gamma function, and $K_{\nu}$ is the modified Bessel function of the second kind. The parameters $\nu$ and $\kappa$ in the Mat\'ern covariance control the smoothness of sample paths and the inverse length scale of the field, while the parameter $\tau$ in the squared exponential covariance controls the length scale. (Note that we are not including the variance parameter that usually appears as a multiplicative constant in the covariances.) Unfortunately, the naive idea of plugging in geodesic distances often leads to failure of positive definiteness of the resulting covariance matrix \cite{gneiting2013strictly,feragen2015geodesic}. 

To circumvent this challenge, the seminal paper \cite{lindgren2011explicit} exploits the stochastic partial differential equation (SPDE) representation of Euclidean GPs with the Mat\'ern covariance function. More precisely, it is shown in \cite{whittle1963stochastic} that the GP with covariance function $c^{\Ma}_{\nu,\kappa}$ over a Euclidean space $\mathbb{R}^m$ is the unique stationary solution to the following equation (up to a multiplicative constant independent of $\kappa$)
\begin{align}
    (\kappa^2-\Delta)^{\frac{\nu}{2}+\frac{m}{4}} u(x) = \kappa^{\nu}\mathcal{W}(x), \quad \quad  x\in\mathbb{R}^m,\label{eq:SPDE}
\end{align}
where $\Delta$ is the usual Laplacian on $\mathbb{R}^m$ and $\mathcal{W}$ is a spatial white noise with unit variance. The equation \eqref{eq:SPDE} can then be lifted to the manifold case to construct Mat\'ern GPs over manifolds \cite{lindgren2011explicit}. Based on this idea, the papers \cite{sanz2020spde,borovitskiy2020matern} study the following series definition of GPs over compact manifolds:
\begin{align}
    (\text{Mat\'ern manifold-GP}) \quad u^{\Ma}&=\kappa^{s-\frac{m}{2}}\sum_{i=1}^{\infty} (\kappa^2+\lambda_i)^{-\frac{s}{2}}\xi_i \psi_i,\quad \quad \xi_i\overset{i.i.d.}{\sim}\mathcal{N}(0,1), \label{eq:continuum Matern}
\end{align}
where $(\lambda_i,\psi_i)$'s are eigenvalue-eigenfunction pairs of the negative Laplace-Beltrami operator $-\Delta_{\mathcal{M}}$ on $\mathcal{M}$. 
Compactness of $\mathcal{M}$ ensures that $\Delta_{\mathcal{M}}$ admits a countable eigenbasis so that the solution to the analog equation of \eqref{eq:SPDE} over $\mathcal{M}$ can be represented as the series \eqref{eq:continuum Matern}. The parameters $s, \kappa>0$ in \eqref{eq:continuum Matern} control the smoothness and the inverse length scale as in the Euclidean case: $s = \nu + m/2$ controls the spectrum decay, while $\kappa$ acts as a cutoff on the essential frequencies. The scaling factor $\kappa^{s-\frac{m}{2}}$ ensures that samples from different $\kappa$'s have $L^2$-norms on the same order (see. e.g. \cite[Remark 2.1]{sanz2020spde}), which is essential in applications where $\kappa$ needs to be inferred.

As the smoothness parameter $\nu\rightarrow\infty$, it can be shown that the Mat\'ern covariance converges (after a suitable normalization) to the SE covariance (see e.g. \cite[Section 4.2]{williams2006gaussian}). Accordingly, there is a similar SPDE to \eqref{eq:SPDE} that characterizes the SE GP on a Euclidean space $\mathbb{R}^m$ \cite{borovitskiy2020matern}: 
\begin{align*}
    e^{-\frac{\tau\Delta}{2}}u(x)=\tau^{\frac{m}{4}}\mathcal{W}(x),\quad \quad x\in\mathbb{R}^m,
\end{align*}
which motivates its manifold analog as the series expansion
\begin{align}
    (\text{SE manifold-GP}) \quad u^{\SE}&=\tau^{\frac{m}{4}}\sum_{i=1}^{\infty}e^{-\frac{\lambda_i\tau}{2}}\xi_i \psi_i,\quad \quad \xi_i\overset{i.i.d.}{\sim}\mathcal{N}(0,1), \label{eq:continuum SE}
\end{align}
where $(\lambda_i,\psi_i)$'s are eigenvalue-eigenfunction pairs of $-\Delta_{\mathcal{M}}$. Here the factor $\tau^{\frac{m}{4}}$ is again interpreted as balancing the magnitude of samples from different $\tau$'s (see Lemma \ref{lemma:SE factor}).
Furthermore, the induced covariance function has the form 
\begin{align}
    c^{\SE}(x,\tilde{x})=\tau^{\frac{m}{2}}\sum_{i=1}^{\infty}e^{-\lambda_i\tau}\psi_i(x)\psi_i(\tilde{x}). \label{eq:continuum SE cf}
\end{align}
Notice that this is also known as the heat kernel (up to the scaling factor $\tau^{m/2}$), which is a natural generalization of the SE kernel over the manifold. A similar expression holds for the induced covariance function of $u^{\Ma}$:
\begin{align}
    c^{\Ma}(x,\tilde{x})=\kappa^{2s-m}\sum_{i=1}^{\infty}(\kappa^2+\lambda_i)^{-s}\psi_i(x)\psi_i(\tilde{x}). \label{eq:continuum Ma cf}
\end{align}

Besides the connection with their Euclidean counterparts, notice that the random fields \eqref{eq:continuum Matern} and \eqref{eq:continuum SE} are series expansions of the eigenfunctions of the Laplace-Beltrami operator, which form an orthonormal basis for $L^2(\mathcal{M})$ and carry rich information about the geometry of $\mathcal{M};$ therefore, \eqref{eq:continuum Matern} and \eqref{eq:continuum SE} are natural GP models for functions over $\mathcal{M}.$ 
However, computing the pairwise covariances \eqref{eq:continuum SE cf} and \eqref{eq:continuum Ma cf} between any two points would require knowledge of the Laplace-Beltrami eigenvalues and eigenfunctions, which are only known analytically for a few manifolds such as the sphere and the torus, and can otherwise be expensive to approximate. More importantly, in applications where only a point cloud representation of $\M$ is available  we need an empirical way to approximate the manifold GPs \eqref{eq:continuum Matern} and \eqref{eq:continuum SE}.  To that end, we will adopt a manifold learning approach using graph Laplacians.  

\subsubsection{GGP Models}\label{sec:graph model}
 The construction in this subsection follows \cite{sanz2020spde}. 
Given a point cloud $\M_N = \{x_1, \ldots, x_N\} \subset \mathcal{M}$, recall that our goal is to build a GP model for $f_N$, the restriction of $f$ to the $x_i$'s. It then suffices to construct an $N$-dimensional Gaussian that approximates the manifold GPs \eqref{eq:continuum Matern} and \eqref{eq:continuum SE}; in particular, we need to construct a suitable covariance matrix. 

To start with, observe that the manifold Mat\'ern GP \eqref{eq:continuum Matern} can be seen as the Karhunen-Lo\`eve expansion of the Gaussian measure \cite{bogachev1998gaussian} (the infinite-dimensional analog of multivariate Gaussian) $\mathcal{N}(0, \mathcal{C})$, where $\mathcal{C}$ is the covariance operator
\begin{align*}
    \mathcal{C}=\kappa^{2s-m}(\kappa^2 I-\Delta_{\mathcal{M}})^{-s},
\end{align*}
 with $I$ denoting the identity operator. 
Therefore a natural candidate for an $N$-dimensional approximation is to consider the multivariate Gaussian $\mathcal{N}(0, \mathcal{C}_N)$, where 
\begin{align}
  \mathcal{C}_N=\kappa^{2s-m}(\kappa^2I_N+\Delta_N)^{-s}  \label{eq:graph covariance matrix} 
\end{align}
for some $\Delta_N\in\mathbb{R}^{N\times N}$ constructed with the $x_i$'s that approximates $-\Delta_{\mathcal{M}}$  with $I_N$ denoting the $N$-dimensional identity matrix. We shall set $\Delta_N$ to be a suitable \emph{graph Laplacian}, as we describe next.

Let $\M_N = \{x_i\}_{i=1}^N$ be a collection of points on $\mathcal{M}$. One can construct a weighted graph over the $x_i$'s by introducing a weight matrix $W\in\mathbb{R}^{N\times N}$ whose entry $W_{ij}$ represents the similarity between points $x_i$ and $x_j$. The \emph{unnormalized graph Laplacian} is then defined as $\Delta_N=D-W$, where $D$ is a diagonal matrix whose entries are $D_{ii}=\sum_{j=1}^N W_{ij}$. One can immediately check that $\Delta_N$ is symmetric and positive semi-definite using the relation
\begin{align*}
    v^{\top} \Delta_N v = \frac12 \sum_{i=1}^N\sum_{j=1}^N W_{ij}|v_i-v_j|^2, \quad v\in \mathbb{R}^N,
\end{align*}
implying that $\Delta_N$ admits a spectral decomposition with nonnegative eigenvalues $\{\lambda_{N,i}\}_{i=1}^N$ (ordered increasingly) and the associated eigenvectors $\{\psi_{N,i}\}_{i=1}^N$ form an orthonormal basis for $\mathbb{R}^N$. Several normalizations of $\Delta_N$ have also been considered, including the \emph{random walk graph Laplacian} $\Delta_N^{\tiny \text{rw}}=D^{-1}\Delta_N$ and \emph{symmetric graph Laplacian} $\Delta_N^{\tiny \text{sym}}=D^{-1/2}\Delta_ND^{-1/2}$, see  \cite{von2007tutorial}. We focus on the unnormalized version due to its symmetry, which makes it a valid choice in the covariance matrix \eqref{eq:graph covariance matrix}, and its convergence properties that we will describe now. 

As its name suggests, $\Delta_N$ approximates the Laplace-Beltrami operator in a suitable sense. Indeed, if we set the pairwise similarity to be 
\begin{align}
    W_{ij}=\frac{2(m+2)}{N\nu_mh_N^{m+2}}\mathbf{1}\{|x_i-x_j|<h_N\}, \label{eq:weight matrix}
\end{align}
where $|\cdot|$ denotes the Euclidean distance, $\nu_m$ is the volume of the $m-$dimensional unit ball and $h_N$ is a graph connectivity parameter, then for suitable choices of $h_N$ it can be shown (see e.g. \cite{garcia2020error} or Proposition \ref{prop:eval efun bound}) that the eigenpair $(\lambda_{N,i},\psi_{N,i})$ of $\Delta_N$ approximates the corresponding eigenpair $(\lambda_i,\psi_i)$ of $-\text{vol}(\mathcal{M})^{-1}\Delta_{\mathcal{M}}$. Based on this fact, we shall now define two GGPs as follows 
\begin{alignat}{3}
    (\text{Matérn GGP})&\quad u_N^{\Ma}&&= \kappa^{s-\frac{m}{2}}\sum_{i=1}^{k_N} (\kappa^2+\lambda_{N,i})^{-\frac{s}{2}}\xi_i \psi_{N,i},\quad \quad &&\xi_i \overset{i.i.d.}{\sim}\mathcal{N}(0,1), \label{eq:graph Matern}\\
    (\text{SE GGP})& \quad u_N^{\SE}&&=\tau^{\frac{m}{4}}\sum_{i=1}^{k_N} e^{-\frac{\lambda_{N,i}\tau}{2}}  \xi_i \psi_{N,i},\quad \quad &&\xi_i \overset{i.i.d.}{\sim}\mathcal{N}(0,1), \label{eq:graph SE}
\end{alignat}
where $k_N\leq N$ is a truncation level to be determined. Notice that Matérn and SE GGPs can be interpreted as discretely indexed GPs over the graph $(\mathcal{M}_N,W)$, hence the name GGP.  Similar objects have also been studied by \cite{sanz2020spde,borovitskiy2021matern,dunson2022graph}. 
When $k_N=N$, we see that \eqref{eq:graph Matern} is nothing but the multivariate Gaussian $\mathcal{N}\bigl(0,\kappa^{2s-m}(\kappa^2I_N+\Delta_N)^{-s} \bigr)$, matching our goal \eqref{eq:graph covariance matrix} at the beginning. The motivation for introducing the truncation is that the spectral approximation accuracy degrades quickly when we go to higher modes (see e.g. Proposition \ref{prop:eval efun bound}), where the error bounds are only meaningful when $h_N\sqrt{\lambda_i}\ll 1$. Therefore \eqref{eq:graph Matern} can be seen as a low rank approximation of \eqref{eq:graph covariance matrix} that keeps only the low and accurate frequencies. By Weyl's law (see e.g. \cite[Theorem 72]{canzani2013analysis}), $\lambda_i\asymp i^{2/m}$ and in particular $\lambda_i\rightarrow \infty$, which suggests a necessary condition $k_N\ll h_N^{-m}$. In Subsection \ref{ssec:truncation} we discuss an empirical way of choosing $k_N$. The induced covariance functions take the form 
\begin{equation}\label{eq:graph cf}
\begin{aligned}
      c_N^{\Ma}(x,\tilde{x})&=\kappa^{2s-m}\sum_{i=1}^{k_N} (\kappa^2+\lambda_{N,i})^{-s}\psi_{N,i}(x)\psi_{N,i}(\tilde{x}), \\
      c_N^{\SE}(x,\tilde{x})&=\tau^{\frac{m}{2}}\sum_{i=1}^{k_N} e^{-\lambda_{N,i}\tau}\psi_{N,i}(x)\psi_{N,i}(\tilde{x}),  
\end{aligned}
\qquad x,\tilde{x}\in\mathcal{M}_N.
\end{equation}

Notice that the definitions \eqref{eq:graph Matern} and \eqref{eq:graph SE} are completely parallel with \eqref{eq:continuum Matern} and \eqref{eq:continuum SE}; hence the spectral convergence of $\Delta_N$ leads to convergence of GGPs to their manifold counterparts. We will rely on the following assumption:

\begin{assumption}\label{assp:manifold}
$\mathcal{M}$ is a smooth, compact and connected submanifold of dimension $m \geq 2$ in $\mathbb{R}^d$ that has no boundary and bounded sectional curvature, normalized so that $\text{vol}(\mathcal{M})=1$. Assume the $x_i$'s are i.i.d. samples from the uniform distribution on $\mathcal{M}$. 
\end{assumption}

The following result provides a simplified statement of the convergence analysis for Mat\'ern GGPs in \cite{sanz2020spde,sanz2020unlabeled} and in addition covers SE GGPs.  The proof can be found in the Appendix \ref{appenA}.
\begin{proposition}\label{prop:graph GP approx bound}
Let $0<\iota<1$ be arbitrary.  Define $\alpha_m=(m+4+\iota)\vee (2m)$  and $\beta_{m,s}=\frac{2s-3m+1}{6m+6} \wedge 1.$ Let $p_m=\frac{3}{4}$ when $m=2$ and $p_m=\frac{1}{m}$ otherwise.
 For $s>\frac{3}{2}m-\frac12$, set
    \begin{align*}
    (\emph{Mat\'ern GGP})\quad \quad & h_N\asymp N^{-\frac{1}{\alpha_m}}(\log N)^{\frac{p_m}{2}}, \quad  k_N\asymp N^{\frac{m\beta_{m,s}}{(2s-3m+1)\alpha_m}}(\log N)^{-\frac{mp_m\beta_{m,s}}{(4s-6m+2)}},\\
    (\emph{SE GGP}) \quad \quad & h_N\asymp N^{-\frac{1}{\alpha_m}}(\log N)^{\frac{p_m}{2}},\quad  (\log N)^{\frac{m}{2}}\ll k_N\ll N^{\frac{m}{(3m+3)\alpha_m}}(\log N)^{-\frac{mp_m}{6m+6}}.
\end{align*}
Under Assumption \ref{assp:manifold}, with probability $1-O(N^{-c})$ for some $c>0$, there exists $T_N:\mathcal{M}\rightarrow \{x_1,\ldots,x_N\}$ satisfying $T_N(x_i)=x_i$ such that 
\begin{align}
\begin{rcases}
    \mathbb{E}\|u_N^{\Ma}\circ T_N-u^{\Ma}\|_{\infty}\lesssim N^{-\frac{\beta_{m,s}}{2\alpha_m}} (\log N)^{\frac{\beta_{m,s}p_m}{4}}\\
    \mathbb{E}\|u_N^{\SE}\circ T_N-u^{\SE}\|_{\infty}\lesssim N^{-\frac{1}{2\alpha_m}}(\log N)^{\frac{p_m}{4}}
\end{rcases}
=:\epsilon_N. \label{eq:epsilon_N}
\end{align}
\end{proposition}
The fact that we can study $L^{\infty}$-norms of these random fields follows from their almost sure continuity established in \cite[Lemma 3]{sanz2020unlabeled} and Lemma \ref{lemma:SE factor}. Proposition \ref{prop:graph GP approx bound} will be a key ingredient in establishing regret bounds for GGP-UCB.

\subsection{Choice of Acquisition Function}\label{ssec:acquisition} 
 When the GGP prior $\pi_N$ matches the truth $f_N$, i.e., when $f_N$ is a sample from $\pi_N$, \cite{srinivas2010gaussian} gives a choice of $B_{N,L}$ for the acquisition function \eqref{eq:acquisition} that ensures vanishing regret. However, this is not necessarily true in our case since $f_N$ is the restriction of a function $f$ over $\mathcal{M}$ whereas the GGP $\pi_N$ is only constructed with $\mathcal{M}_N$. A mismatch is possible and below we address this issue following ideas in \cite{bogunovic2021misspecified}.

 Suppose that the function $f$ to be optimized is a sample from the manifold GP \eqref{eq:continuum Matern} (or \eqref{eq:continuum SE}) and we adopt the corresponding GGP prior $\pi_N$ given by \eqref{eq:graph Matern} (resp. \eqref{eq:graph SE}) for $f_N$.  
Proposition \ref{prop:graph GP approx bound} then imples that if $u_N\sim \pi_N$, we have
with probability $1-\delta$
\begin{align}
   \|u_N-f_N\|_{\infty} \leq \delta^{-1} \epsilon_N, \label{eq:misspec bound}
\end{align}
where here $\|\cdot\|_{\infty}$ denotes the entry-wise maximum and $\epsilon_N$ is a placeholder for the approximation error defined in \eqref{eq:epsilon_N}. In other words, there is potentially a \emph{misspecification error} coming from the fact that we are using an approximate GP to model $f_N$. With the understanding of such error obtained in Proposition \ref{prop:graph GP approx bound}, we can follow the approach in \cite{bogunovic2021misspecified} and set 
\begin{align}
    B_{N,\ell}=\sqrt{2\log \left(\frac{\pi^2\ell^2N}{6\delta}\right)}+\frac{\epsilon_N\sqrt{\ell-1}}{\delta\sigma},\label{eq:Bl}
\end{align}
where we recall that $\sigma$ is the noise standard deviation. 
Notice that this differs from the plain GP-UCB in  \cite{srinivas2010gaussian} by the additional term $\epsilon_N\sqrt{\ell-1}/\delta\sigma$ that aims to correct for the misspecification. Intuitively, such correction leads to an increase of the weight on the posterior standard deviation, which accounts for the increased uncertainty due to the approximate modeling. Therefore at the $\ell$-th iteration, we shall pick the candidate $z_{\ell}$ as 
\begin{align}
    z_{\ell}=\underset{z\in\M_N}{\operatorname{arg\,max}}\,\, \biggl\{ \mu_{N,\ell-1}(z)+\left[\sqrt{2\log \left(\frac{\pi^2\ell^2N}{6\delta}\right)}+\frac{\epsilon_N\sqrt{\ell-1}}{\delta\sigma}\right]\sigma_{N,\ell-1}(z) \biggr\},  \label{eq:choice of zt}
\end{align}
where $\mu_{N,\ell-1}$ and $\sigma_{N,\ell-1}$ are defined as in \eqref{eq:pm pstd} but with $c_N(\cdot,\cdot)$ being the graph covariance functions \eqref{eq:graph cf}.

\begin{remark}
In our setting we do not have access to the underlying manifold $\M$ and hence continuous optimization is not applicable.
As a result, \eqref{eq:choice of zt} is optimized over the discrete set $\mathcal{M}_N$ and would require evaluation of the acquisition function over the entire point cloud. If $N$ is large and evaluating the acquisition function over the full point cloud is costly, then one can, for practical purposes,   approximately  optimize \eqref{eq:choice of zt}  using  a subsample of the point cloud $\M_N$. Optimizing the acquisition function approximately is common practice in BO. 
It is important to emphasize, however, that in the applications that motivate our work the objective function $f$ is much more expensive to evaluate than the acquisition function.  $\hfill \square$
\end{remark}

\subsection{Main Result: Regret Bounds}\label{ssec:regret bound}

Now we are ready to state our main result. Its proof can be found in Appendix \ref{appenB}.
\begin{theorem}\label{thm:regret bound}
Suppose $f$ is a sample from the Mat\'ern manifold-GP \eqref{eq:continuum Matern} with parameters $\kappa,s$ (resp. SE manifold-GP \eqref{eq:continuum SE} with parameter $\tau$). Let $\pi_N$ be the Mat\'ern (resp. SE) GGP constructed as in Proposition \ref{prop:graph GP approx bound} with the same parameters. Apply Algorithm \ref{algo:GP-UCB} with $\pi_N$ and with $B_{N,\ell}$ given by \eqref{eq:Bl}. Under Assumption \ref{assp:manifold}, for $N$ large enough, we have with probability $1-2\delta-O(N^{-c})$ that 
\begin{align*}
     r_{N,L}\leq C\left[\frac{\sqrt{2\log (\pi^2L^2N/6\delta)}}{\sqrt{L}}+\frac{\epsilon_N}{\delta\sigma}\right] \sqrt{k_N\log L} \, ,\qquad \forall \,\, L\geq 1,
\end{align*}
where $c,C>0$ are universal constants. Here we recall that $\sigma$ is the observation noise standard deviation, $k_N$ is the truncation parameter in Proposition \ref{prop:graph GP approx bound}, and $\epsilon_N$ is the approximation error as in \eqref{eq:epsilon_N}. 
\end{theorem}

\begin{remark}\label{remark:L<<N}
By plugging the scaling in Proposition \ref{prop:graph GP approx bound}, we get
\begin{align}\label{eq:regret explicit rate}
    r_{N,L}=\widetilde{O}\Big((L^{-\frac12}+\epsilon_N)\sqrt{k_N}\Big) 
    = \widetilde{O}\begin{cases}
    L^{-\frac12}N^{\frac{m\beta_{m,s}}{(4s-6m+2)\alpha_m}}+N^{-\frac{(2s-4m+1)\beta_{m,s}}{(4s-6m+2)\alpha_m}} &\text{(Mat\'ern)}\\
    L^{-\frac12}+N^{-\frac{1}{2\alpha_m}}&\text{(SE)}
    \end{cases}.
\end{align}
Here the notation $\widetilde{O}(\cdot)$ means that we have dropped all dependence on logarithmic factors. The regret goes to zero as both $N$ and $L$ approach infinity in both cases (when $s>\frac{7}{4}m+\frac12$ for the Mat\'ern case), although we recall that $N$ should be treated as a fixed large number and $L$ is the running index. The two terms in the above upper bound can be understood as the error incurred by Bayesian optimization and by misspecification, respectively. For a fixed $N$, the regret will decrease as $L\rightarrow\infty$ to a threshold imposed by the misspecification error, which itself will go to zero with more data points from $\mathcal{M}$ as $N\rightarrow\infty$. Notice that the two terms are balanced at $L\asymp N^{\beta_{m,s}/\alpha_m}$ for the Mat\'ern case and $L\asymp N^{1/\alpha_m}$ for the SE case. Since $\beta_{m,s}\leq 1$, for a fixed large enough $N$, number of queries of the order $L \ll N$ would be sufficient in both cases because otherwise the error coming from misspecification will dominate. We shall demonstrate by simulations in Section \ref{sec:simulation} that the algorithm is able to find the optimizer (or an almost optimizer) after a number $L$ of queries that is significantly smaller than the size $N$ of the point cloud.   $\hfill \square$
\end{remark}

We end this section with a bound on the continuum regret $r_{N,L}^{\tiny \operatorname{cont}}$ (see its definition in \eqref{eq:continuum regret}). 

\begin{corollary}\label{cor:continuum regret}
Under the same assumptions as in Theorem \ref{thm:regret bound}, $r_{N,L}^{\tiny \operatorname{cont}}$ follows the same bound as \eqref{eq:regret explicit rate}.  
\end{corollary}
Therefore we can recover a global maximizer of $f$ over $\mathcal{M}$ as both $N$ and $L$ tend to infinity while keeping $L\ll N$.  
 
\section{Estimation and Tuning of GGP-UCB Parameters}\label{sec:implementation}
This section discusses important considerations for the practical implementation of the GGP-UCB algorithm. Subsections \ref{ssec:MLE}, \ref{ssec:truncation} and \ref{ssec:tuningB} describe respectively the estimation of prior GGP parameters, the choice of graph connectivity $h_N$ and truncation level $k_N$, and the empirical tuning of the acquisition function.

\subsection{Parameter Estimation}\label{ssec:MLE}
Theorem \ref{thm:regret bound} holds under the assumption that the GGP model uses the same parameters $\kappa,s,\tau$ as those for the truth. However, these parameters are typically unavailable in practice and need to be estimated. In this subsection we give a possible empirical solution. 

Recall that at the $\ell$-th iteration we pick the next query point $z_\ell$ based on \eqref{eq:choice of zt} and observe a noisy function value 
\begin{align*}
    y_\ell=f(z_\ell)+\eta_\ell,
\end{align*} 
where $f$ is assumed to be a sample from the manifold GP \eqref{eq:continuum Matern} or \eqref{eq:continuum SE} with parameter $\theta$ ($\theta=(\kappa,s)$ for the Mat\'ern case and $\theta=\tau$ for the SE case). We shall obtain an estimate $\theta_\ell$ of $\theta$ in each iteration of the above procedure using a maximum likelihood estimation approach:  
\begin{align}
    \theta_\ell=\underset{\theta}{\operatorname{arg\,max}}\,\, \mathbb{P}(Y_\ell\,|\,\theta), \label{eq:ML theta_t}
\end{align}
where $Y_\ell=(y_1,\ldots,y_\ell)^\top$. Exact maximization of \eqref{eq:ML theta_t} would require knowing the covariance structure of the underlying manifold GP, in particular the eigenpairs of the Laplace-Beltrami operator, the lack of which is precisely the reason why we introduced our graph-based approach. However, since the GGPs \eqref{eq:graph Matern} and \eqref{eq:graph SE} are what we actually use for modeling $f$, a natural idea is then to seek for parameters of these surrogate models that can best fit the data. Therefore we shall consider the following ``surrogate'' data model by pretending that the $y_\ell$'s are generated from the GGPs:
\begin{align*}
\begin{split}
    y_k &= u_N(z_k) + \eta_k, \quad \quad \eta_k \overset{i.i.d.}{\sim} \Nc(0, \sigma^2),\qquad k=1,\ldots,\ell,\\
      u_N &\sim \mathcal{N}(0, \mathcal{C}_N^{\theta}), 
     \end{split}
\end{align*}
where $\mathcal{C}_N^{\theta}$ is the covariance matrix associated with \eqref{eq:graph cf}. It follows that 
\begin{align}
    Y_{\ell}\sim \mathcal{N}(0,\Sigma^{\theta}_N), \quad \quad \Sigma^{\theta}_N = A \mathcal{C}_N^{\theta}A^\top+\sigma^2 I_\ell, \label{eq:likelihood of Y}
\end{align}
where $A\in\mathbb{R}^{\ell\times N}$ is a matrix of 0's and 1's whose entries indicate the indices of the $z_\ell$'s among $\M_N = \{x_i\}_{i=1}^N$. Maximization of the likelihood of $Y_{\ell}$ under \eqref{eq:likelihood of Y} gives the estimate $\theta_{\ell}$.

\subsection{Determining the Truncation Level \texorpdfstring{$\boldsymbol{k_N}$}{Kn} and the Graph Connectivity \texorpdfstring{$\boldsymbol{h_N}$}{hn}}
\label{ssec:truncation}
As mentioned in Subsection \ref{sec:graph model}, the truncation level $k_N$ is crucial in that the higher frequencies obtained from the graph Laplacian give poor approximations to their manifold counterparts and can have a negative impact on approximating manifold GPs. Proposition \ref{prop:graph GP approx bound} gives a scaling for $k_N$ that is based on the asymptotic behavior of the graph Laplacian. Empirically, one can simply choose $k_N$ by plotting the spectrum of $\Delta_N$. 

\begin{figure}[!htb]
\centering
\minipage{0.333\textwidth}
  \includegraphics[width=\textwidth]{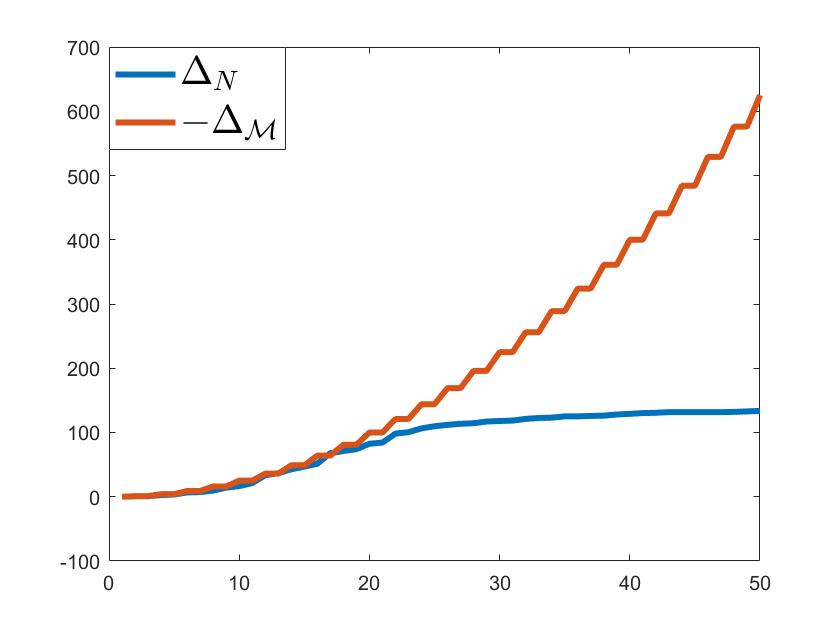} 
\vspace{-10pt}
\endminipage
\minipage{0.333\textwidth}
  \includegraphics[width=1\textwidth]{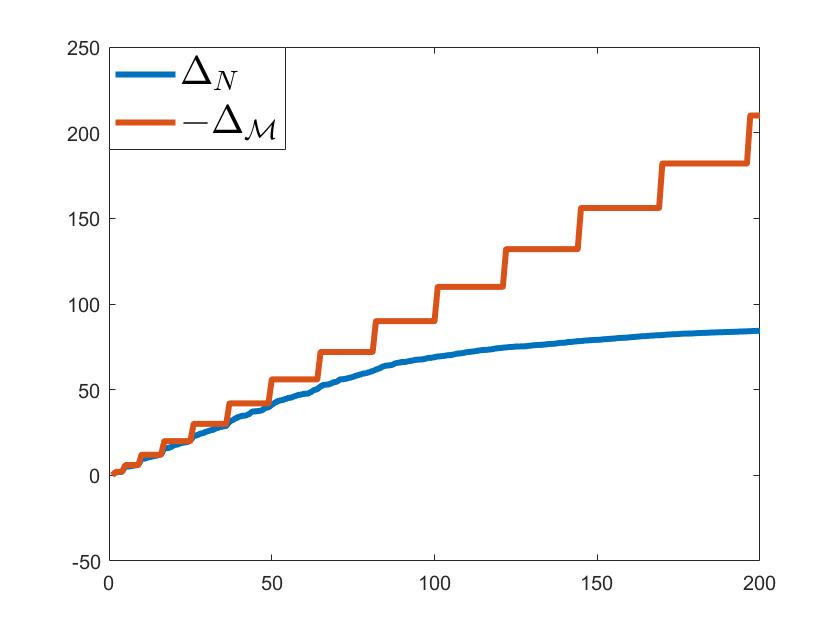}
\vspace{-10pt}
\endminipage
\caption{Spectrum of $\Delta_N$ versus spectrum of $-\Delta_{\mathcal{M}}$ for the unit circle (left) and the unit sphere (right). }
\label{fig:spectra}
\end{figure}

Proposition \ref{prop:eval efun bound} in the appendix gives an upper bound on the eigenvalue approximation, where the error is small only when $h_N\sqrt{\lambda_i}\ll1$. In practice, what we usually observe is not only such poor spectral approximation for large $i$'s, but also a ``saturation'' of the graph Laplacian eigenvalues after certain threshold. Figure \ref{fig:spectra} shows the first several eigenvalues of the Laplace-Beltrami operator $-\Delta_\M$ and the graph Laplacian $\Delta_N$  constructed with (a) $N=500$ points from the unit circle; and (b) $N=3000$ points from the unit sphere. We can see from both plots that for small index $i,$ the eigenvalues $\lambda_{N,i}$ of $\Delta_N$ approximate well the eigenvalues $\lambda_i$ of $-\Delta_\M$; however, the spectrum of $\Delta_N$ is essentially flat for large $i$. Therefore one can choose $k_N$ to be around the point of saturation in the spectrum of $\Delta_N$. 
Such saturation phenomenon, noted for instance in \cite{trillos2020consistency,garcia2018continuum,ruiyilocalregularization}, also helps to explain the need for truncation. Indeed, the eigenvalues $\lambda_{N,i}$ determine the decay of the coefficients in the series \eqref{eq:graph Matern} and \eqref{eq:graph SE} defining our GGPs. Without a truncation, 
too much weight would be given to the high frequencies, which would lead to overly rough sample paths. 
  
Another key parameter in the construction of our GGPs is the graph connectivity $h_N$ in the definition of the weights \eqref{eq:weight matrix}. A common choice \cite{garcia2020error,sanz2020unlabeled} is $h_N\propto\sqrt{\rho_N}$, where $\rho_N$ defined in \eqref{eq:rho_N} can be interpreted as the maximum distance between any two nearby $x_i$'s or the ``resolution'' of $\mathcal{M}_N$. In particular, the choice $h_N\propto\sqrt{\rho_N}$ ensures that the neighborhood of each $x_i$ in the graph is local but rich enough to capture the local geometry. Moreover, this choice balances the two terms in the error bound $\rho_N/h_N+h_N\sqrt{\lambda_i}$ in Proposition \ref{prop:eval efun bound}. The scaling of $\rho_N$ is shown in \cite[Theorem 2]{garcia2020error} and recorded in Proposition \ref{prop:d infinity bound}, which leads to the choice $h_N=C N^{-1/2m}$ (the logarithmic factor can be absorbed into the proportion constant). The proportion constant can be determined again by plotting the corresponding spectrum of $\Delta_N$. Starting with a large $C$, one can keep decreasing the value of $C$ while observing the point of saturation becoming larger, until one hits a point where the spectrum is no longer meaningful. This latter case will happen when $h_N$ is too small so that the graph is disconnected and the graph Laplacian has repeated zero eigenvalues.

\subsection{Empirical Tuning of the Acquisition Function}\label{ssec:tuningB}
Recall that the selection rule \eqref{eq:choice of zt} incorporates information on the level of misspecification $\epsilon_N$ incurred by the GGPs. Proposition \ref{prop:graph GP approx bound} gives such a bound on $\epsilon_N$, which goes to zero as $N\rightarrow \infty$. However, for practical considerations, the upper bound may not be small for certain ranges of $\delta$ and $N$, especially since there is a possibly non-sharp proportion constant in $\epsilon_N$. 
Therefore this could cause the term $\epsilon_N\sqrt{\ell-1}/\delta\sigma$ in $B_{N,\ell}$ to be overly large, so that the acquisition function puts too much weight on the posterior standard deviation, as a result of which exploration overwhelms exploitation. For this reason, we shall consider instead setting $B_{N,\ell}$ as
\begin{align}
    B_{N,\ell} = a \, \sqrt{2\log \left(\frac{\pi^2\ell^2N}{6\delta}\right)} \,\,, \label{eq:modofied Bl}
\end{align}
with a tuning parameter $a>0$.
As noticed in the simulation studies in \cite{srinivas2010gaussian}, setting $a=1/5$ in practice leads to the best performance in well-specified cases, i.e., when $\epsilon_N=0$ (although their theoretical results are proved for $a=1$). Motivated by this observation, we shall set $a=1/2$ throughout for our simulation studies in Section \ref{sec:simulation} to account for the case $\epsilon_N\neq 0$. The idea is that the original correction term $\epsilon_N\sqrt{\ell-1}/\delta\sigma$ for misspecification is now absorbed as the increment $(1/2-1/5)\sqrt{2\log (\pi^2\ell^2N/6\delta)}$. 

Finally, the selection rule \eqref{eq:choice of zt} searches for the query points over the entire $\mathcal{M}_N$ at each iteration, which could return points that have already been picked and get stuck at local optima in practice. We shall modify \eqref{eq:choice of zt} slightly by maximizing it over $\mathcal{M}_N\backslash\{z_1,\ldots,z_{\ell-1}\}$ at the $\ell$-th iteration, i.e., by asking the algorithm to output a query point that has not been chosen in previous iterations.

\section{Numerical Examples} \label{sec:simulation}
This section contains preliminary numerical experiments that complement the theory. The main focus will be to illustrate the performance of our method within the scope of Bayesian optimization rather than conduct an exhaustive comparison with existing discrete optimization algorithms. 

In Subsection \ref{sec:ex-circle} we give a detailed investigation of our approach over the unit circle, where eigenvalues and eigenfunctions of the Laplace-Beltrami operator are analytically known and manifold GPs are computable. The goal of this example is to show that our discrete GGP-UCB algorithm, which only requires point cloud data from the unit circle, achieves comparable performance to a UCB algorithm with manifold GPs. We also illustrate the parameter estimation technique discussed in Subsection \ref{ssec:MLE}.  In Subsection \ref{sec:ex-cow} we consider an artificial manifold for which the spectrum of its Laplace-Beltrami operator is not available, showcasing a typical application of our framework when the manifold is only accessed through a point cloud. The goal of this example is to show the empirical advantage of using our geometry-informed GGPs over Euclidean GPs. Finally, in Subsection \ref{sec:ex-heat} we apply Algorithm \ref{algo:GP-UCB} to solve an inverse problem ---heat source detection over the sphere, which is only represented as a point cloud. Here the objective function is defined in terms of a partial differential equation that needs to be numerically solved. 
The goal of this example is to illustrate the applicability of our algorithm with expensive-to-evaluate objective functions that need to be approximated using graph-based techniques. 

Throughout all the examples in Subsections \ref{sec:ex-circle} and \ref{sec:ex-cow}, we set $\sigma=0.05\cdot \|f_N\|_2/\sqrt{N}$, which corresponds to a noise level of roughly 5\%. We adopt the selection rule \eqref{eq:modofied Bl} and set $\delta=0.1$ in the choice of $B_{N,\ell}$. 

\subsection{The Unit Circle}\label{sec:ex-circle}
Let $\mathcal{M}$ be the unit circle in $\mathbb{R}^2$ and $\mathcal{M}_N=\{x_i\}_{i=1}^{N=500}$ be i.i.d. samples from the uniform distribution over $\mathcal{M}$. 
 The fact that the eigenvalues and eigenfunctions of the Laplace-Beltrami operator are available in closed form allows us to carry out ---for comparison purposes--- computation on the continuum level. In particular, we can compute the manifold GP covariance functions defined in \eqref{eq:continuum SE cf} and \eqref{eq:continuum Ma cf}.

To start with, suppose first that $f$ is a sample from the manifold Mat\'ern GP \eqref{eq:continuum Matern} with parameters $\tau_*$ and $s_*$, which can be generated from \eqref{eq:continuum Matern} with a sufficiently high truncation. We shall compare the performance of Algorithm \ref{algo:GP-UCB} with three different choices of the prior: (i) \eqref{eq:continuum Matern} with true parameters, (ii) \eqref{eq:graph Matern} with true parameters, and (iii) \eqref{eq:graph Matern} with inferred parameters, i.e., 
\begin{alignat}{4}
    &\text{(MGP-UCB)} \quad  &&u^{ \mathcal{M}}&&=\kappa_*^{s_*-\frac{m}{2}}\sum_{i=1}^K(\kappa_*^2+\lambda_i)^{-\frac{s_*}{2}}\xi_i\psi_i, \qquad &&\xi_i\overset{i.i.d.}{\sim}\mathcal{N}(0,1), \qquad  \label{eq:MGP}\\
    &\text{(GGP-UCB)} \quad &&u^{\mathcal{M}_N}&&=\kappa_*^{s_*-\frac{m}{2}}\sum_{i=1}^{k_N}(\kappa_*^2+\lambda_{N,i})^{-\frac{s_*}{2}}\xi_i\psi_{N,i},  &&\xi_i\overset{i.i.d.}{\sim}\mathcal{N}(0,1), \label{eq:graph trun}\\
    &\text{(GGP-UCB-ML)} \quad &&u^{ \text{\tiny MLE}}&&=\kappa_{\ell}^{s_{\ell}-\frac{m}{2}}\sum_{i=1}^{k_N}(\kappa_{\ell}^2+\lambda_{N,i})^{-\frac{s_{\ell}}{2}}\xi_i\psi_{N,i},   &&\xi_i\overset{i.i.d.}{\sim}\mathcal{N}(0,1), \label{eq:graph mle}
\end{alignat} 
where $K=100$ is a truncation for computing $u^{\mathcal{M}}$, and $\kappa_{\ell}$ and $s_{\ell}$ are the estimated parameters as discussed in Subsection \ref{ssec:MLE}. Specifically, we shall view MGP-UCB as an oracle algorithm whose performance serves as a benchmark, since for the graph-based algorithms we assume to be only given the point cloud $\{x_i\}_{i=1}^{N=500}$ and to have no access to the $\lambda_i$'s and $\psi_i$'s. %
We set $k_N=20$ and $h_N=4\times N^{-1/2}$ in the construction of $\Delta_N$.

\begin{figure}[!htb]
\centering
\minipage{0.333\textwidth}
  \includegraphics[width=\textwidth]{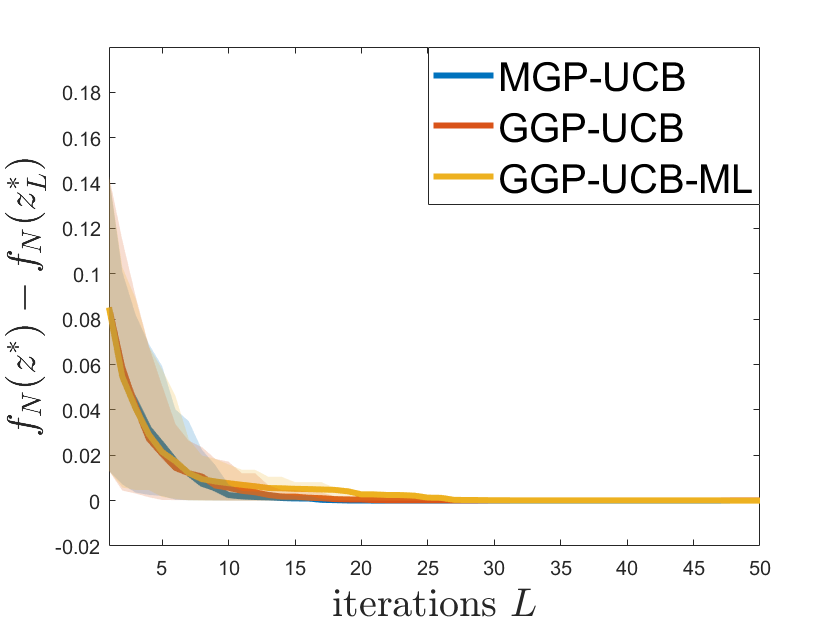}
\vspace{-10pt}\subcaption{$\kappa^2_*=5,s_*=2.$}
\endminipage
\minipage{0.333\textwidth}
  \includegraphics[width=1\textwidth]{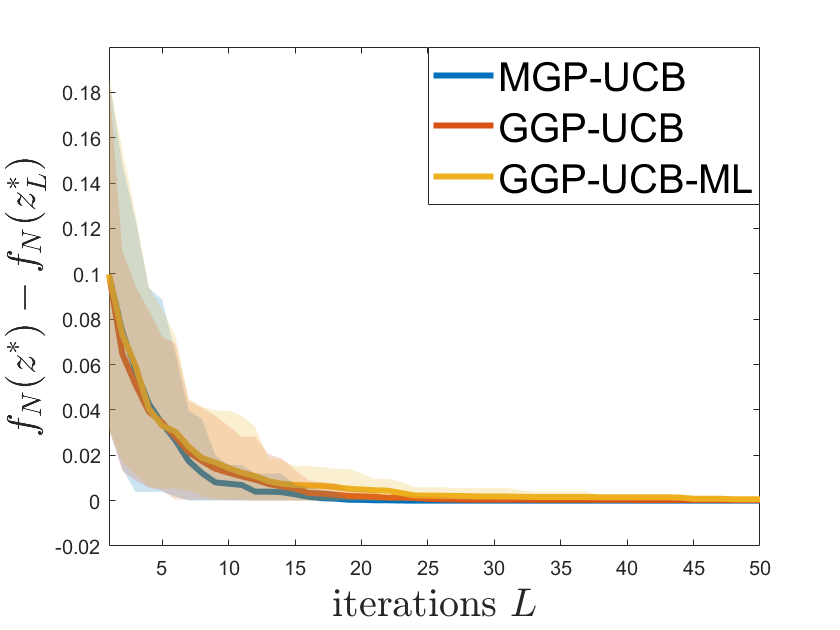}
\vspace{-10pt}\subcaption{$\kappa^2_*=10,s_*=2.$}
\endminipage
\minipage{0.333\textwidth}
  \includegraphics[width=1\textwidth]{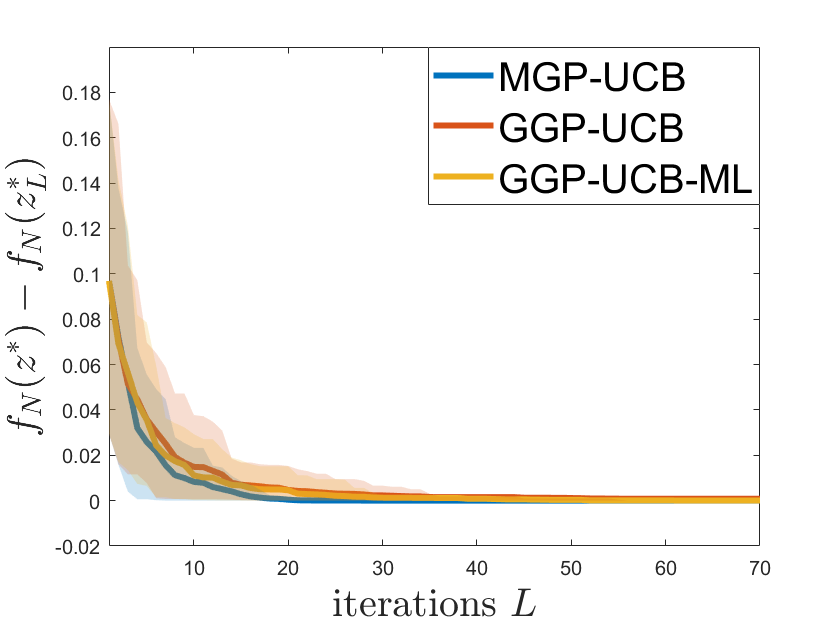}
\vspace{-10pt}\subcaption{$\kappa^2_*=15,s_*=2.$}
\endminipage

\caption{Comparisons of the simple regrets obtained from MGP-UCB (prior with \eqref{eq:MGP}), GGP-UCB (prior with \eqref{eq:graph trun}), and GGP-UCB-ML (prior with \eqref{eq:graph mle}) when $f$ is a Mat\'ern GP \eqref{eq:continuum Matern}.   The curves represent the average regrets over 50 trials and the shaded regions represent the 10\% $\sim$ 90\% percentiles.  }
\label{fig:circle-ma}
\end{figure}

Figure \ref{fig:circle-ma} shows the comparison for three sets of parameters $(\kappa_*,s_*)$, representing increasingly oscillatory true objective $f$. 
In all cases, the oracle MGP-UCB approach achieves the smallest regret, which is expected since it assumes complete knowledge of the unit circle. Meanwhile, the other two approaches show competitive performance and find the maximizer in less than $L=50$ iterations, which is much smaller than the size $N=500$ of the point cloud. In particular, incorporating maximum likelihood estimation of the parameters gives similar performance compared to the case when the parameters are assumed to be known. 
In a parallel setting, we also perform a similar comparison when the truth is a SE GP \eqref{eq:continuum SE}, where the graph SE GP \eqref{eq:graph SE} is used for modeling. Figure \ref{fig:circle-SE} shows the comparison, which is qualitatively similar to the Mat\'ern case except that the approach incorporating maximum likelihood gives a slightly worse performance. Nevertheless, it is still able to find a near optimizer within 50 iterations. 
 
\begin{figure}[!htb]
\minipage{0.333\textwidth}
  \includegraphics[width=\textwidth]{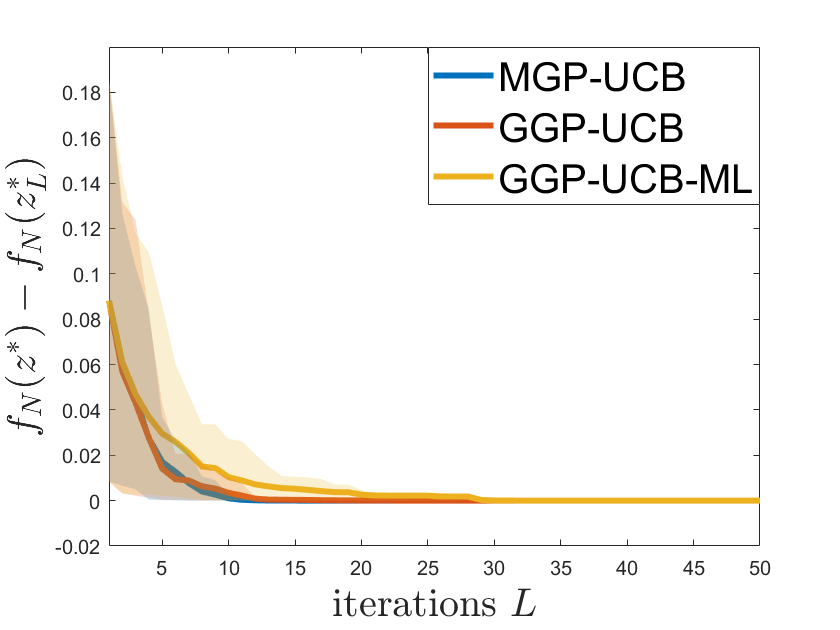}
\vspace{-10pt}\subcaption{$\tau_*=0.15.$}
\endminipage
\minipage{0.333\textwidth}
  \includegraphics[width=\textwidth]{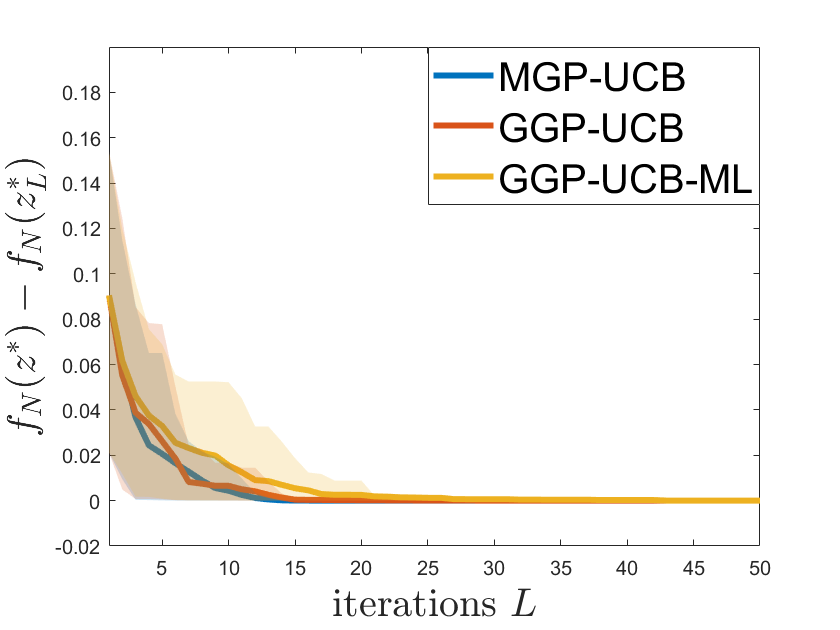}
\vspace{-10pt}\subcaption{$\tau_*=0.1.$}
\endminipage
\minipage{0.333\textwidth}
  \includegraphics[width=1\textwidth]{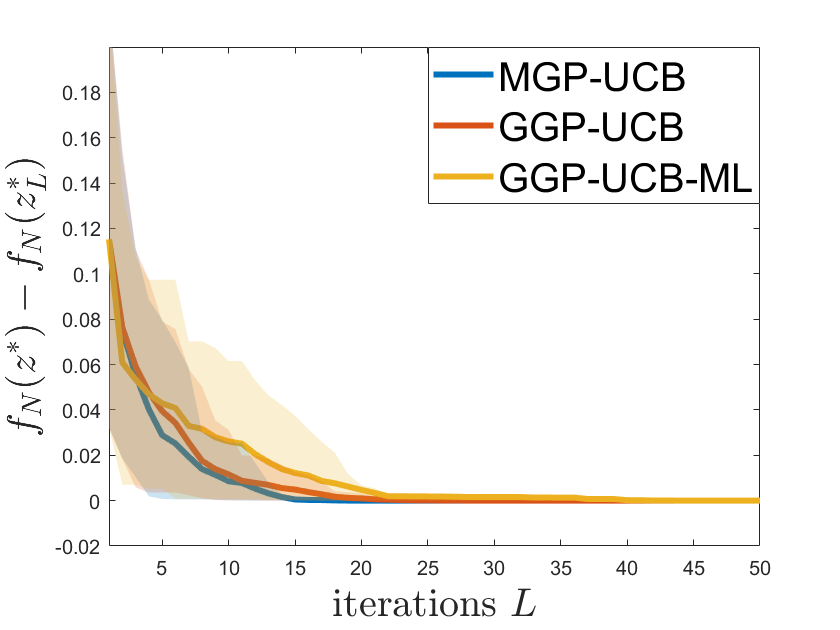}
\vspace{-10pt}\subcaption{$\tau_*=0.05.$}
\endminipage
\caption{Comparisons of the simple regrets obtained from MGP-UCB (prior with \eqref{eq:MGP}), GGP-UCB (prior with \eqref{eq:graph trun}), and GGP-UCB-ML (prior with \eqref{eq:graph mle}) when $f$ is a SE GP \eqref{eq:continuum SE}.   The curves represent the average regrets over 50 trials and the shaded regions represent the 10\% $\sim$ 90\% percentiles.  }
\label{fig:circle-SE}
\end{figure}

Next, we investigate the effect of the number $N$ of point cloud samples on the algorithmic performance. We generate the truth from \eqref{eq:MGP} as before and apply our graph-based algorithms with $N$=100, 300, 500 uniform samples from the unit circle. 
Figure \ref{fig:circle-point cloud} shows the results, suggesting improved performance as $N$ increases, in agreement with the qualitative behavior predicted by our regret bounds in \eqref{eq:regret explicit rate}.

\begin{figure}[!htb]
\centering
\minipage{0.333\textwidth}
  \includegraphics[width=\textwidth]{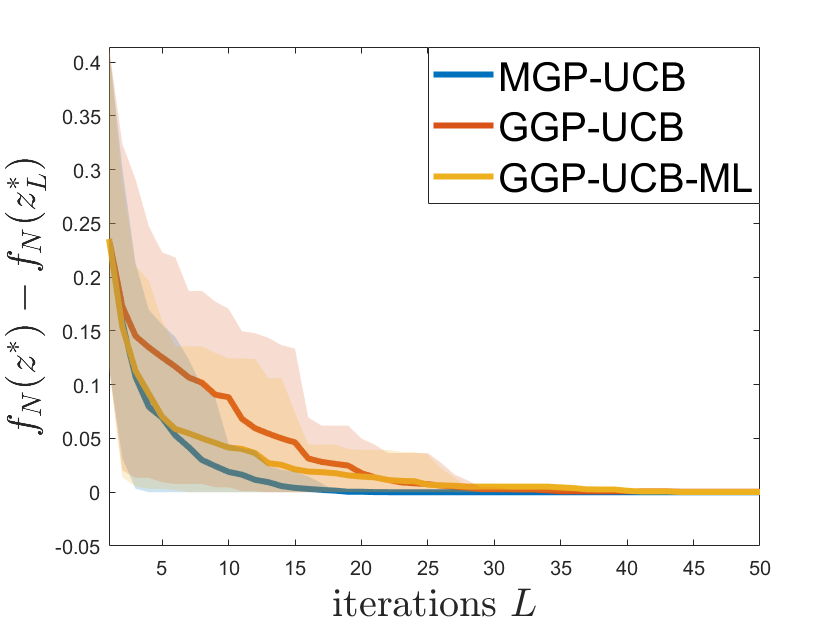}
\vspace{-10pt}\subcaption{$N=100$.}
\endminipage
\minipage{0.333\textwidth}
  \includegraphics[width=1\textwidth]{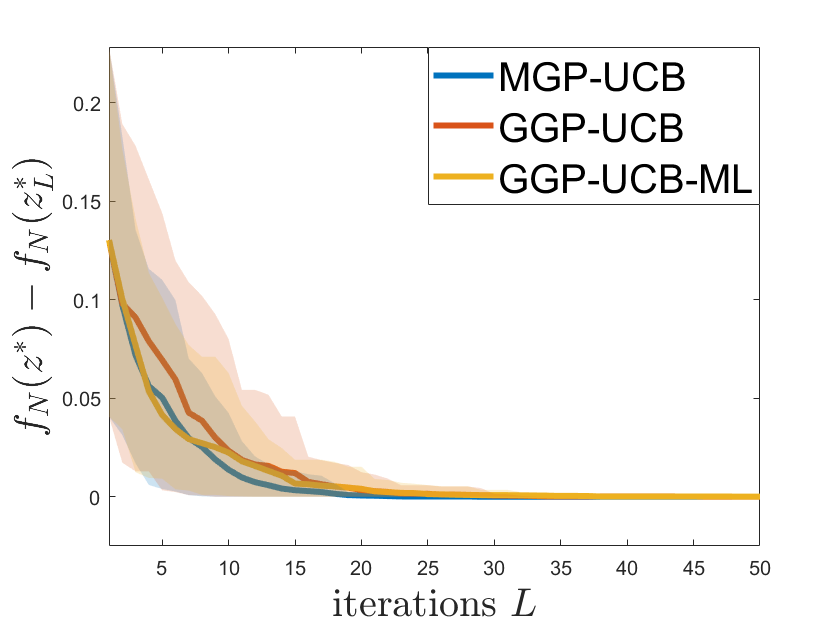}
\vspace{-10pt}\subcaption{$N=300$.}
\endminipage
\minipage{0.333\textwidth}
  \includegraphics[width=1\textwidth]{circle-tau=15-s=2.png}
\vspace{-10pt}\subcaption{$N=500$.}
\endminipage

\caption{Comparisons of the simple regrets obtained from MGP-UCB (prior with \eqref{eq:MGP}), GGP-UCB (prior with \eqref{eq:graph trun}), and GGP-UCB-ML (prior with \eqref{eq:graph mle}) with different size $N$ of the point cloud when $f$ is a Mat\'ern GP \eqref{eq:continuum Matern} with parameters $\kappa_*^2=15$ and $s_*=2$.   The curves represent the average regrets over 50 trials and the shaded regions represent the 10\% $\sim$ 90\% percentiles.  }
\label{fig:circle-point cloud}
\end{figure}

Finally, to further investigate the performance of our GGP-UCB algorithm, we consider optimizing three benchmark functions ---the Levy, Ackley, and Rastrigin functions defined over the circle (with suitable rescaling): 
\begin{align}
    \flevy(\theta)&=\Big(\frac{3\theta}{4}\Big)^2 \bigg(1+\sin^2\Big(\frac{\pi(3\theta+3)}{2}\Big)\bigg)  \tag{Levy}\label{eq:levy},\qquad \theta\in[-\pi,\pi),\\
    \fackley(\theta)&=-20\exp(-0.1\theta)-\exp(\cos(2\pi \theta) )+20+\exp(1),\qquad \theta\in[-\pi,\pi),\tag{Ackley}\label{eq:ackley}\\
    \frastri(\theta)&=2+\theta^2-2\cos(2\pi\theta),\qquad \theta\in[-\pi,\pi),\tag{Rastrigin}\label{eq:rastri}
\end{align}
where we identify points on the circle with their angle $\theta \in [-\pi,\pi)$.
The top row of Figure \ref{fig:circle-benchmarks} shows plots of the functions $\flevy$, $\fackley$, and $\frastri$, all of which admit many sharp local minima. These benchmark functions will serve as examples where the truth to be optimized is not generated from a GP. 
As before, we shall compare the performance of Algorithm \ref{algo:GP-UCB} with three different choices of prior \eqref{eq:MGP}, \eqref{eq:graph trun}, \eqref{eq:graph mle}, by manually setting $\kappa_*=15$ and $s_*=1$ for the first two. 
The results are shown in the bottom row of Figure \ref{fig:circle-benchmarks}, where all algorithms can find the global optimizer with very few iterations
(much fewer than the total number $N=500$ of the point cloud), including GGP-UCB-ML which infers the covariance parameters. 
This illustrates the applicability of our algorithm when the truth is not necessarily a sample path from the same GP model that we use for the algorithm.

\begin{figure}[!htb]
\centering
\minipage{0.333\textwidth}
  \includegraphics[width=\textwidth]{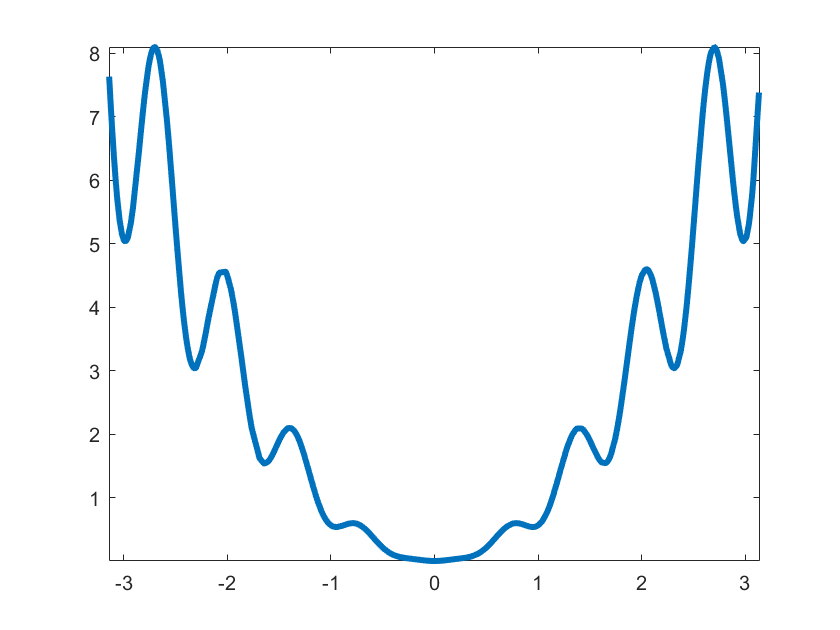}
\vspace{-10pt}
\endminipage
\minipage{0.333\textwidth}
  \includegraphics[width=\textwidth]{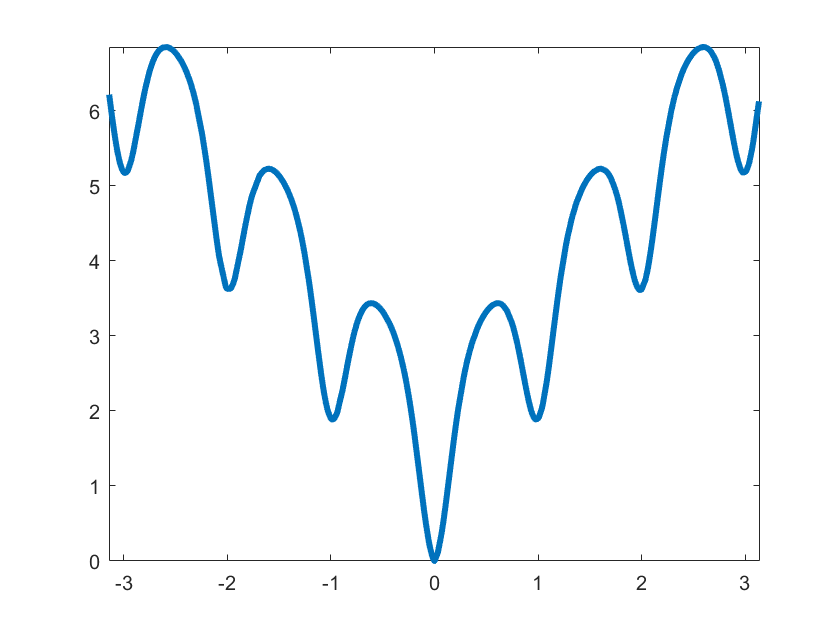}
\vspace{-10pt}
\endminipage\minipage{0.333\textwidth}
  \includegraphics[width=\textwidth]{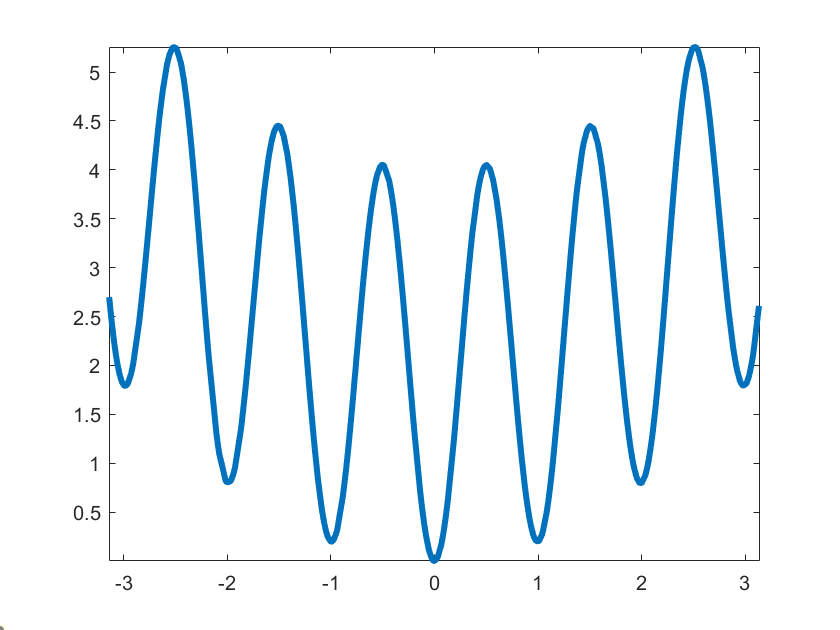}
\vspace{-10pt}
\endminipage

\minipage{0.333\textwidth}
  \includegraphics[width=1\textwidth]{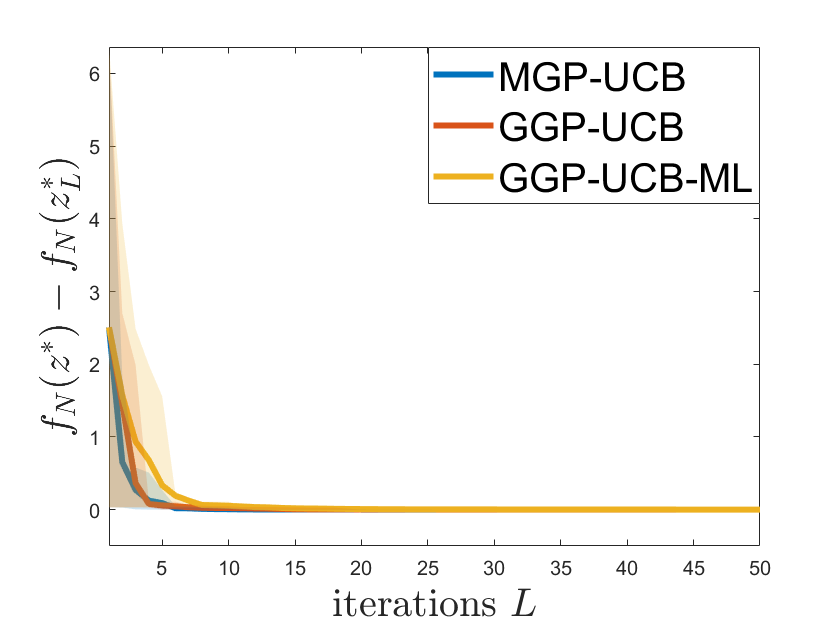}
\vspace{-10pt}\subcaption{\ref{eq:levy}.}
\endminipage
\minipage{0.333\textwidth}
  \includegraphics[width=1\textwidth]{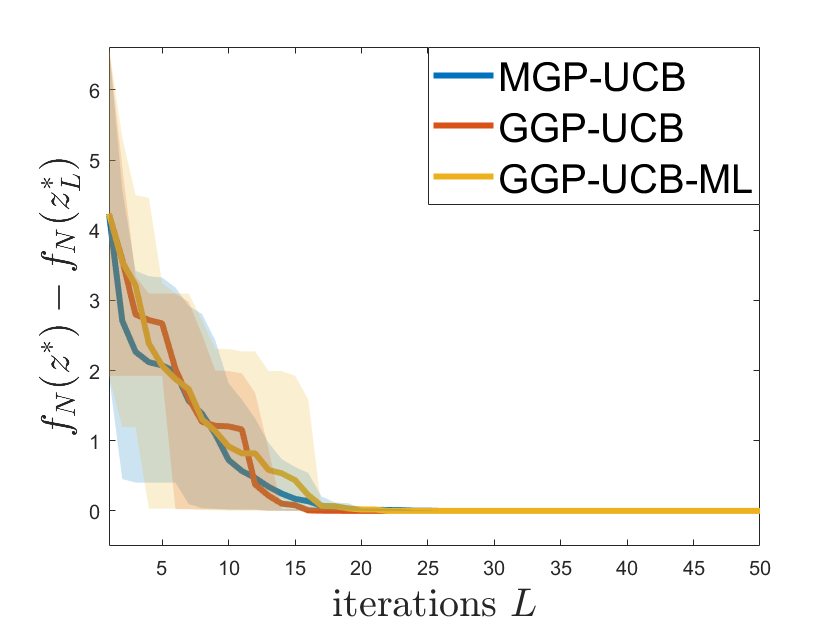}
\vspace{-10pt}\subcaption{\ref{eq:ackley}.}
\endminipage
\minipage{0.333\textwidth}
  \includegraphics[width=1\textwidth]{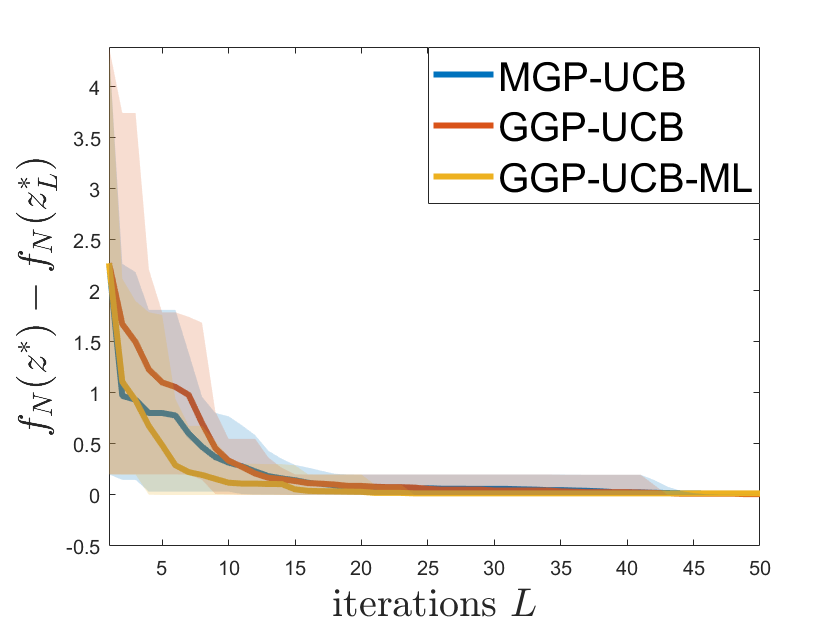}
\vspace{-10pt}\subcaption{\ref{eq:rastri}.}
\endminipage

\caption{Top row: plots of the \ref{eq:levy}, \ref{eq:ackley}, and \ref{eq:rastri} functions. Bottom row: Comparisons of the simple regrets obtained from MGP-UCB (prior with \eqref{eq:MGP}), GGP-UCB (prior with \eqref{eq:graph trun}), and GGP-UCB-ML (prior with \eqref{eq:graph mle}) for optimizing the three functions respectively. The curves represent the average regrets over 50 trials and the shaded regions represent the 10\% $\sim$ 90\% percentiles.  }
\label{fig:circle-benchmarks}
\end{figure}

\begin{remark}\label{remark:mle}
We end this example with a remark on inferring the GP parameters with maximum likelihood. 
For the Mat\'ern case, our experience suggests that joint estimation of $\kappa_{\ell}, s_{\ell}$ for \eqref{eq:graph mle} turns out to be unstable, and hence in the simulations above we have fixed $\kappa_{\ell}$ to be 1 throughout and only estimated $s_{\ell}$.
Such an observation may be related to the fact that not all parameters for the Mat\'ern model but only a certain combination of them are identifiable (see e.g. \cite{zhang2004inconsistent,bolin2020rational,li2021inference}). 
This issue may be exacerbated by the fact that the graph Mat\'ern GP we adopt is only an approximation of the Mat\'ern model, and similarly for the squared exponential model.  A detailed investigation of maximum likelihood for GGPs is an interesting direction for future research. Our focus on the remaining experiments will be however on illustrating other important aspects of our GGP-UCB algorithm, and for this reason we henceforth assume the GP parameters to be known or tune them empirically. $\hfill \square$
\end{remark}

\subsection{Two-Dimensional Artificial Manifold}\label{sec:ex-cow}
In this subsection we consider an artificial two-dimensional manifold, whose  
point cloud representation---taken from Keenan Crane’s 3D repository \cite{3dmodel}---is shown in Figure \ref{fig:cow-manifold}. This example is motivated by an application to locate the point of highest temperature \cite{srinivas2010gaussian} on a surface where an explicit parameterization is not given. 
Unlike the unit circle case in Subsection \ref{sec:ex-circle}, the eigenvalues and eigenfunctions of the Laplace-Beltrami operator over this new manifold are no longer known analytically, which prevents us from computing manifold GP covariances. The goal of this example is to demonstrate the superior performance of GGPs over Euclidean GPs.  

First, we shall generate our truth using a finer point cloud than the one given for optimization. More precisely, the original dataset $\mathcal{M}_{\bar{N}}$ provided by \cite{3dmodel} consists of $\bar{N}=2930$ points, but we only assume to be given a subsample of $N=2000$ points as our $\mathcal{M}_N$. 
The truth is then generated as a sample defined on the finer point cloud $\mathcal{M}_{\bar{N}}$:
\begin{align}\label{eq:cow-truth}
    f_{\bar{N}}=\kappa_*^{s_*-\frac{m}{2}}\sum_{i=1}^{k_{\bar{N}}}(\kappa_*^2+\lambda_{\bar{N},i})^{-\frac{s_*}{2}}\xi_i\psi_{\bar{N},i},  \qquad \xi_i\overset{i.i.d.}{\sim}\mathcal{N}(0,1),
\end{align}
where $\lambda_{\bar{N},i}$ and $\psi_{\bar{N},i}$'s are the eigenpairs of the graph Laplacian $\Delta_{\bar{N}}$ constructed with all $\bar{N}$ points. Here the graph connectivity is taken to be $h_{\bar{N}}=4\times \bar{N}^{-1/2}$ and $k_{\bar{N}}$ is set to be 50 based on the eigenvalue saturation of $\Delta_{\bar{N}}$. Figure \ref{fig:cow-ma-sample} shows one realization of $f_{\bar{N}}$ with parameters $\kappa_*^2=5$ and $s_*=2.5$

Since the manifold GP is not available in this example, we shall compare the performance of Algorithm \ref{algo:GP-UCB} with prior taken as a GGP (cf.\eqref{eq:graph trun} with graph connectivity $h_N=4\times N^{-1/2}$ and truncation $k_N=50$) or a Euclidean GP (EGP). As the truth \eqref{eq:cow-truth} is of Mat\'ern type, it is natural to take the EGP as defined by the usual Mat\'ern covariance function \eqref{eq:Ma SE cf} by viewing points in $\mathcal{M}_N$ as elements of $\mathbb{R}^3$. 
As discussed in Remark \ref{remark:mle}, we shall use the true parameters in GGP modeling, but point out that the true parameters are not necessarily the ones that lead to the best performance since the truth is generated based on $\Delta_{\bar{N}}$, whose eigenpairs are only close to but different from those of $\Delta_N$ used for computation. 
For EGP modeling, we tune the parameters empirically and report the one that leads to the smallest regret. The results are presented in Figure \ref{fig:cow-ma-regret}, suggesting that GGP modeling outperforms EGP and can find the optimizer with far fewer queries than the size $N=2000$ of the given point cloud. 
In a parallel setup, Figure \ref{fig:cow-se-regret} compares the performance of GGP-UCB with EGP-UCB when the truth and the associated prior models are of squared exponential type (cf. \eqref{eq:graph SE} and \eqref{eq:Ma SE cf}), where qualitatively similar behavior is observed.

\begin{figure}[!htb]
\minipage{0.333\textwidth}
  \includegraphics[width=\textwidth]{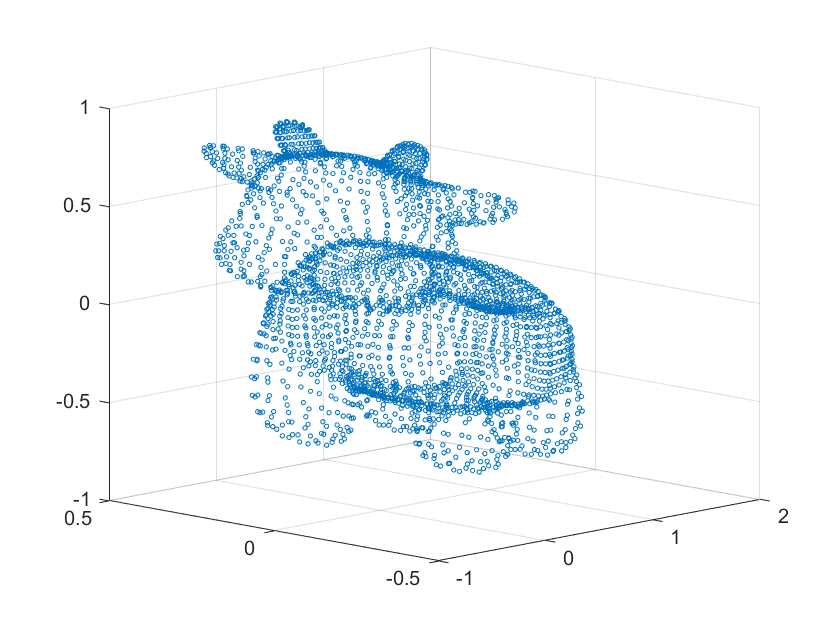}
\vspace{-10pt}\subcaption{Point cloud.}\label{fig:cow-manifold}
\endminipage\hfill
\minipage{0.333\textwidth}
  \includegraphics[width=1\textwidth]{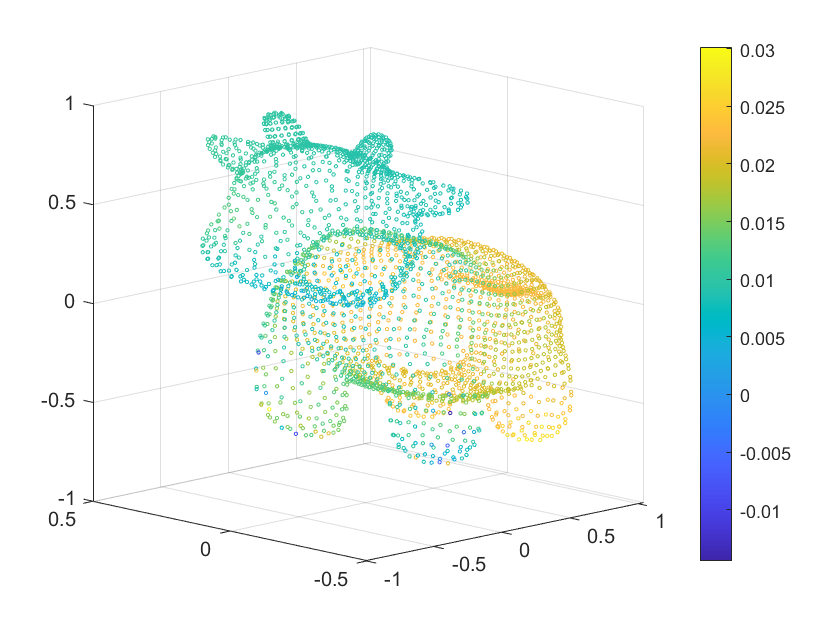}
\vspace{-10pt}\subcaption{Mat\'ern GGP sample.}\label{fig:cow-ma-sample}
\endminipage\hfill
\minipage{0.333\textwidth}
  \includegraphics[width=1\textwidth]{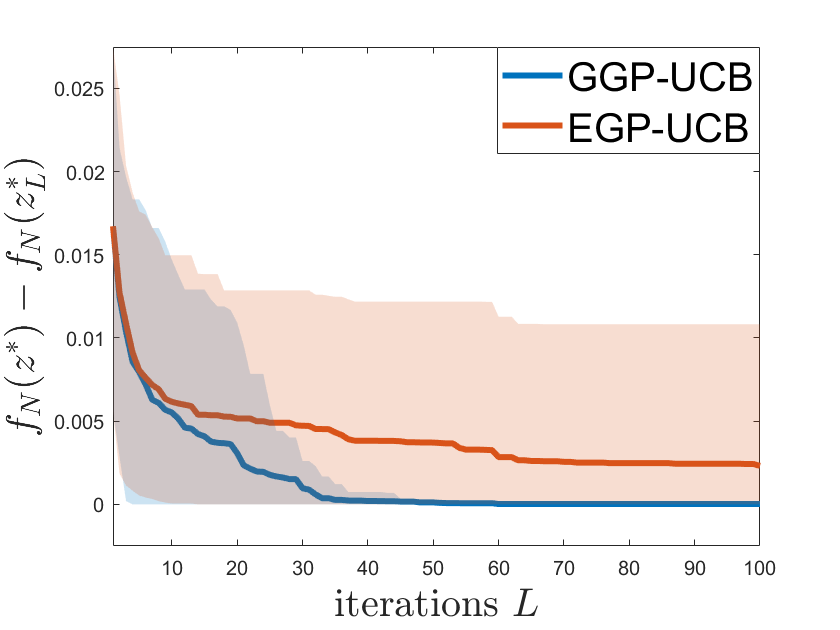}
\vspace{-10pt}\subcaption{Simple regrets.}\label{fig:cow-ma-regret}
\endminipage\hfill

\caption{
 (a) Point cloud. (b) A random sample  $f_{\bar N}$ defined as \eqref{eq:cow-truth} with $\kappa^2_*=5,s_*=2.5$; values of $f_
 {\bar N}$ vary smoothly along the point cloud. (c) Comparison of simple regrets as a function of $L$ between GGP-UCB and EGP-UCB.   The curves represent the average regrets over 50 trials and the shaded regions represent the 10\% $\sim$ 90\% percentiles. }
\label{fig:cow}
\end{figure}

\begin{figure}[!htb]
\centering
\minipage{0.333\textwidth}
  \includegraphics[width=1\textwidth]{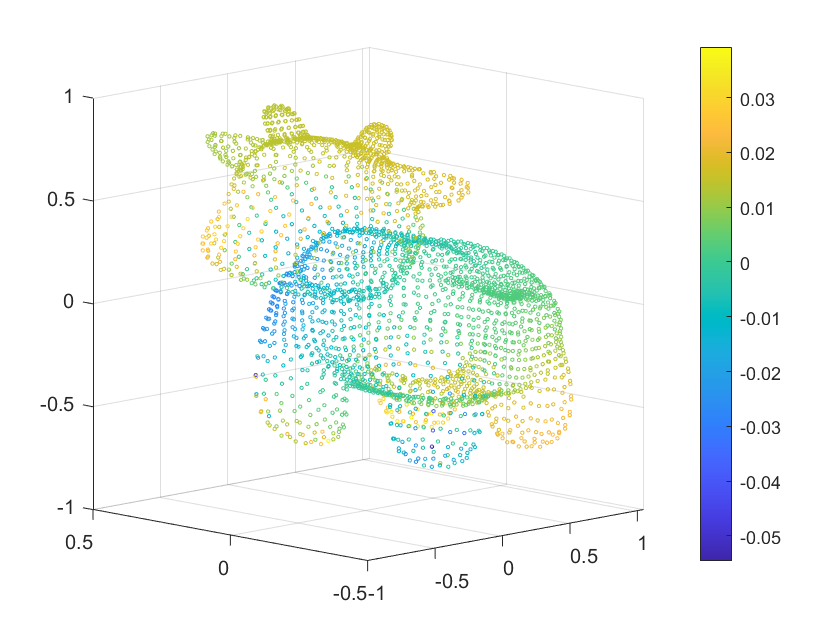}
\vspace{-10pt}\subcaption{SE GGP sample.}\label{fig:cow-se-sample}
\endminipage
\minipage{0.333\textwidth}
  \includegraphics[width=1\textwidth]{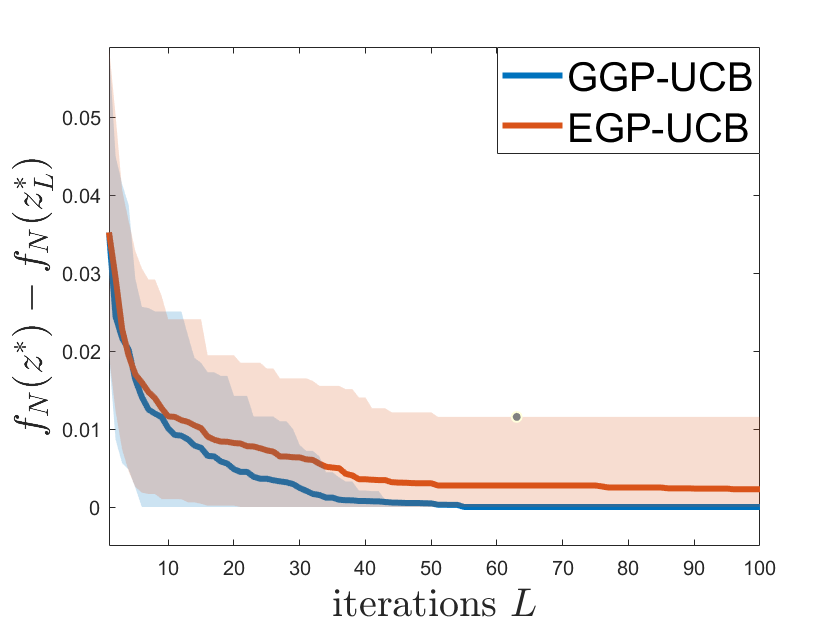}
\vspace{-10pt}\subcaption{Simple regrets.}\label{fig:cow-se-regret}
\endminipage 
\caption{
 (a) A random sample from \eqref{eq:graph SE} based on the graph Laplacian $\Delta_{\bar{N}}$ with $\tau_*=0.05$.   (b) Comparison of simple regrets as a function of $L$ between GGP-UCB and EGP-UCB.   The curves represent the average regrets over 50 trials and the shaded regions represent the 10\% $\sim$ 90\% percentiles.  }
\label{fig:cow-se}
\end{figure}

\subsection{Heat Source Detection on the Sphere}\label{sec:ex-heat}
 In this subsection we employ Algorithm \ref{algo:GP-UCB} on a heat source detection problem on the two-dimensional unit sphere $S^2$, which is given only as a point cloud. The goal of this example is to demonstrate the applicability of our BO framework in inverse problem settings, where the objective function to be optimized involves an expensive-to-evaluate \emph{forward map} that usually needs to be approximated.

Consider the heat equation 
\begin{align}\label{heatEqn}
\begin{split}
\begin{cases}
\phi_t  = \Delta_{S^2} \phi , \quad \quad \quad \quad &(x, t) \in S^2\times[0,\infty),  \\
\phi(x,0) = \phi_0(x),   &x \in S^2,
\end{cases}
\end{split}
\end{align}
where $\Delta_{S^2}$ is the Laplace-Beltrami operator on $S^2$ and $\phi_0$ is an initial heat configuration. The solution of the heat equation for some time $t > 0$ is given by
\begin{equation}
    \phi(x, t) = \sum_{i=1}^\infty \langle \phi_0, \psi_i \rangle_{S^2} \cdot e^{-\lambda_i t}\psi_i(x), \quad x \in S^2, \label{eq:heat sol}
\end{equation}
where $\{(\lambda_i, \psi_i)\}_{i=1}^\infty$ are the eigenpairs of $-\Delta_{S^2}$ and $\langle \cdot, \cdot \rangle_{S^2}$ is the Riemannian  inner product associated to $S^2$. The initial heat configuration is given by 
\begin{align}\label{eq:heat source}
\phi_0(x) = \exp\left({\zeta {z^*}^\top x}\right), \quad \zeta > 0, ~x \in S^2,
\end{align}
which can be viewed as an unnormalized density of the von-Mises Fisher distribution \cite{fisher1953dispersion} on $S^2$. 
 A larger concentration parameter $\zeta$ leads to more probability mass centered around its mean $z^*.$

\begin{figure}[!htb]
\minipage{0.333\textwidth}
  \includegraphics[width=\textwidth]{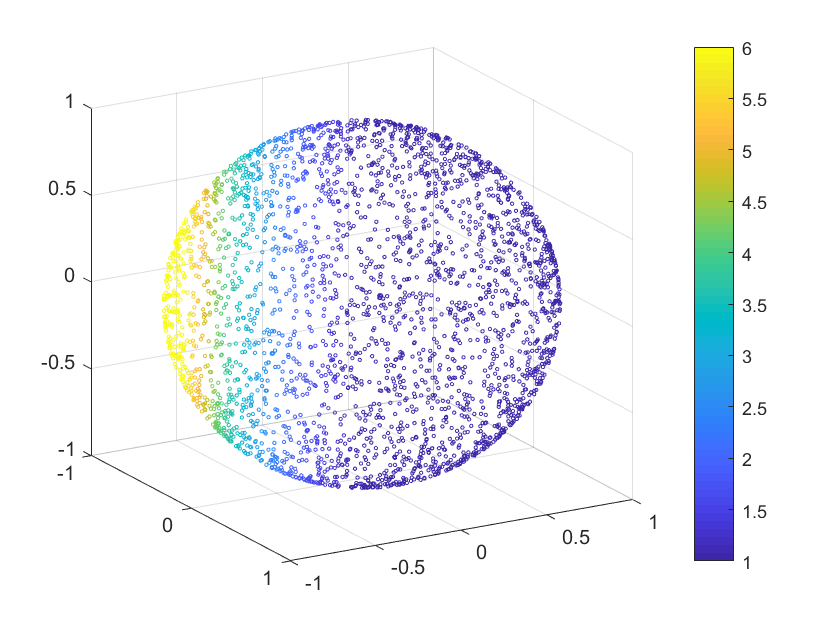}\subcaption{$t=0.$}
\label{fig:init_heat}
\endminipage\hfill
\minipage{0.333\textwidth}
  \includegraphics[width=1\textwidth]{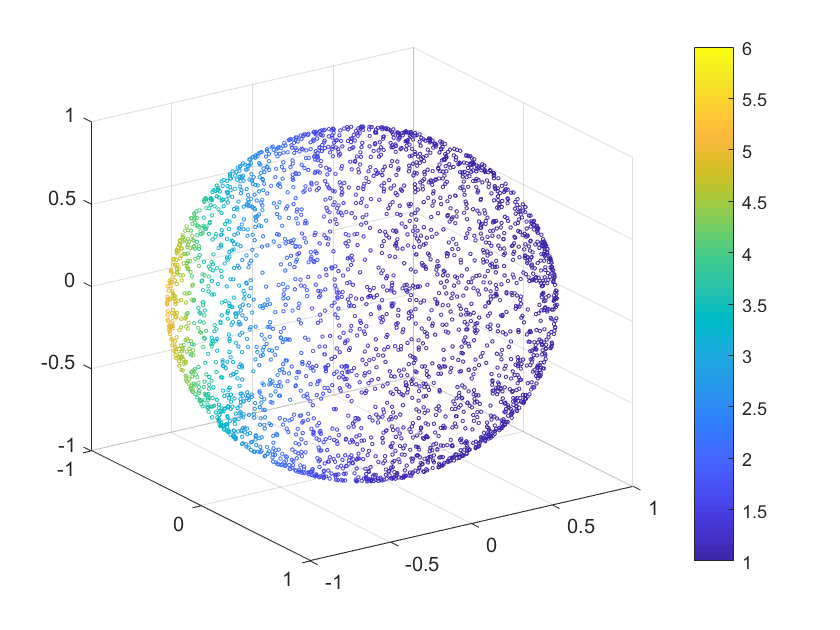}\subcaption{$t=0.25.$}
\label{fig:HEAT_NOISE_T025}
\endminipage\hfill
\minipage{0.333\textwidth}
  \includegraphics[width=1\textwidth]{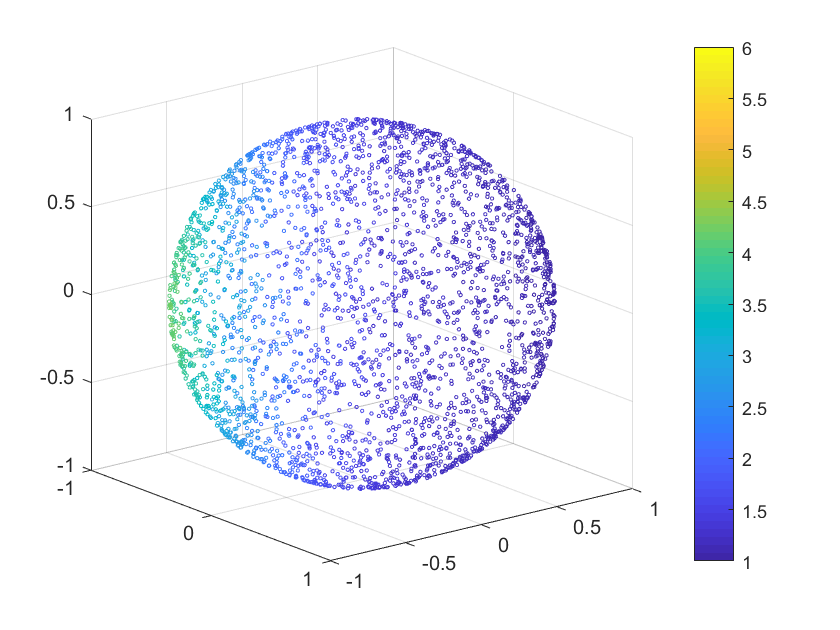}\subcaption{$t=0.4.$}
\label{fig:HEAT_NOISE_T04}
\endminipage\hfill

\caption{(a) Initial heat over the point cloud.
(b) Noisy evaluation of heat at $t = 0.25$. (c) Noisy evaluation of heat at $t = 0.4$. %
}
\label{fig:heat}
\end{figure}
Our goal is to recover the center $z^*$ of the initial heat configuration,  assuming we are only given a point cloud $\mathcal{M}_N=\{x_i\}_{i=1}^{N=3000}$ but not $\mathcal{M}$ directly, and noisy heat measurements at some positive time $t$ of the form 
$$
\mathsf{d} = \phi(\mathbf{x}) + \eta, \quad \quad \phi(\mathbf{x}) = \bigl( \phi(x_1,t), \ldots, \phi(x_N,t) \bigr)^\top, 
$$
where $\eta \sim \mathcal{N}(0, 0.01 I_N)$.  To generate $\phi(x,t)$, we truncate \eqref{eq:heat sol} at $i=36$, by keeping only the terms with $\lambda_i\leq 30$ (the sixth repeated eigenvalue of $-\Delta_{S^2}$).  Figure \ref{fig:heat} contains plots of an example of initial heat configuration with $\zeta=2$ and the corresponding noisy data for times $t = 0.25$ and $t = 0.4$.  
Assuming that the center $z^* \in \M_N$, we adopt an optimization perspective to this inverse problem \cite{sanzstuarttaeb} and attempt to maximize the objective function
$$
f(z) = -\log \|\mathsf{d} - \mathcal{G}(z) \|_{\infty}, \quad \quad z \in \M_N,
$$
 along the point cloud $\M_N$, where $\mathcal{G}(z) \in \R^N$ is the forward map given by 
\begin{align}
    [\mathcal{G}(z)]_k=\sum_{i=1}^{\infty}\langle \phi_0^z,\psi_i \rangle_{S^2}\cdot e^{-\lambda_it}\psi_i(x_k),\qquad x_k\in\mathcal{M}_N, 
\end{align}
with $\phi^z_0(x) = \exp\left({\zeta z^\top x}\right)$ for $x \in S^2$. 
However, since $\mathcal{M}$ is only known through $\mathcal{M}_N$, the eigenvalue and eigenfunctions should be also treated as unknown to us. Therefore, we shall instead maximize the approximate objective function 
$$
f_N(z) = -\log \|\mathsf{d} - \mathcal{G}_N(z) \|_\infty, \quad \quad z \in \M_N,
$$
where
$$
\mathcal{G}_N(z) = \sum_{i=1}^{k_N}\langle \phi^z_{0, N}, \psi_{N, i} \rangle \cdot e^{-\lambda_{N, i} t}\psi_{N, i}, \quad \phi^z_{0, N} =\Bigl( \exp\bigl({\zeta z^\top x_1}\bigr), \cdots, \exp\bigl({\zeta z^\top x_{N}}\bigr)\Bigr)^\top
$$
with the hope that the optimizer of $f_N$ agrees with, or at least is close to, that of $f$.
Here, as before,  $\{(\lambda_{N, i}, \psi_{N, i})\}_{i=1}^N$ are eigenpairs of the unnormalized graph Laplacian and $\langle \cdot, \cdot \rangle$ is the standard Euclidean inner product.  For the truncation level, we set $k_N=70$ to account for the discrepancy ---shown in Figure \ref{fig:spectra}--- between the spectrum of the graph Laplacian and that of the negative Laplace-Beltrami operator.

 To optimize $f_N,$ we apply Algorithm \ref{algo:GP-UCB} with a graph Mat\'ern prior \eqref{eq:graph trun} with parameters $s=4$, $\kappa=1$. There is no observation noise in this case since $f_N$ can be computed exactly, so that $\mu_{N,\ell}$ and $\sigma_{N,\ell}$ in the acquisition function will be computed using \eqref{eq:pm pstd} with $\sigma=0$ and $\smash{Y_\ell = (f_N(z_1), \cdots, f_N(z_\ell))^\top}$. Since we are interested in the recovery of $z^*$, we shall report the distance measure $\|z^*-z_L^*\|_{2}$, where $z_L^*$ is the query point returned by GGP-UCB or random sampling that maximizes $f_N$ in the first $L$ iterations. 
The results are shown in Figure \ref{fig:heat_dist} for observations $\mathsf{d}$ collected at two different times $t=0.25$ and $t=0.4$. Qualitatively similar performance as in previous examples is achieved. 
However, notice that in Figure \ref{fig:mean_dist_t025_matern} the recovery is not exact, as the distance $\|z^*-z_L^*\|_2$ does not decrease to zero. This is because we are searching for the maximizer of the approximate objective $f_N$, which differs from the true heat source $z^*$ when $t$ is large due to the approximation error of $\mathcal{G}_N$ to $\mathcal{G}$. In other words, the attainable discrepancy, defined as the distance between $z^*$ and the maximizer of $f_N$, is nonzero in this case. Besides this effect caused by an error in the approximation of the objective, the simulation results suggest that our GGP-UCB algorithm correctly finds the maximizer of the approximate objective $f_N$ with a significantly smaller number $L$ of queries than the total number $N$ of points in $\M_N$.

\begin{figure}[!htb]
\centering
\minipage{0.333\textwidth}
  \includegraphics[width=\textwidth]{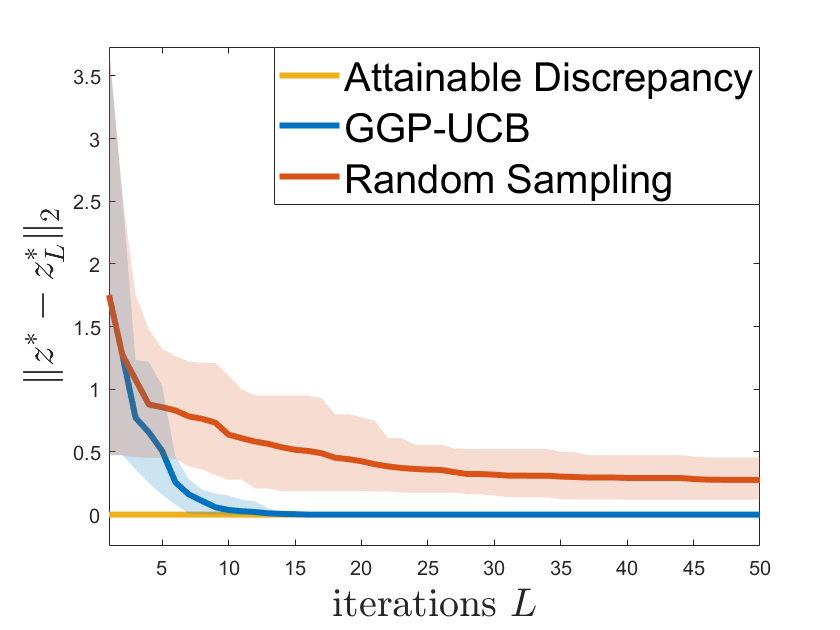}\subcaption{$t=0.25.$}
\label{fig:mean_dist_t01_matern}
\endminipage
\minipage{0.333\textwidth}
  \includegraphics[width=1\textwidth]{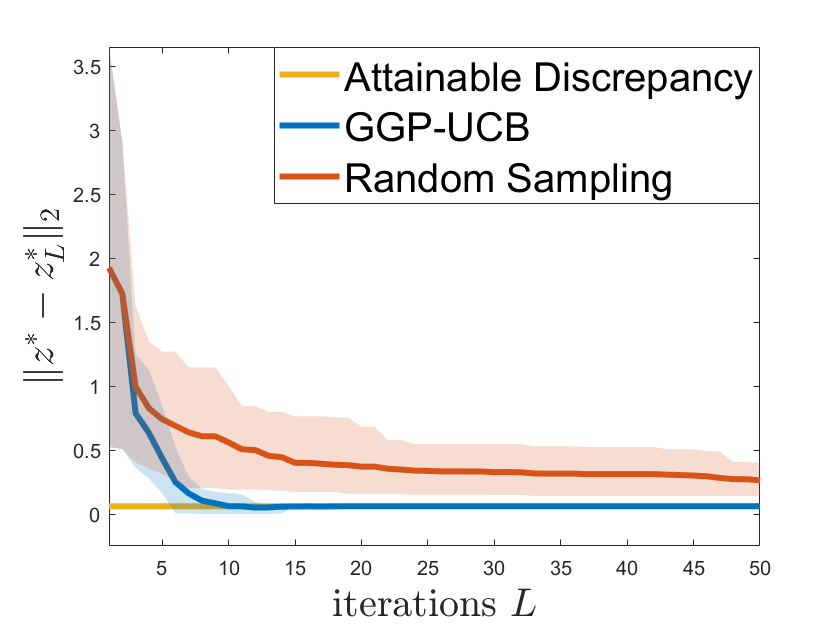}\subcaption{$t=0.4.$}
\label{fig:mean_dist_t025_matern}
\endminipage
\caption{Recovery error $\|z^*-z_L^*\|_2,$ where $z^*$ is the true source in \eqref{eq:heat source}
and $z_L^*$ is the query point returned by GGP-UCB or random sampling that maximizes $f_N$ in the first $L$ iterations. 
Heat measurements are collected at times (a) $t = 0.25$ and (b) $t = 0.4$.
The curves represent the average regrets over 50 trials and the shaded regions represent the 10\% $\sim$ 90\% percentiles.
}
\label{fig:heat_dist}
\end{figure}

We remark that there are two intertwined aspects which make source detection difficult for intermediate to large $t$ values.
The smoothing effect of the forward map $\mathcal{G}$ implies that a larger observation time will lead to a more flattened (homogeneous) temperature configuration, as shown in Figure \ref{fig:heat}.  
In other words, two rather different initial heat configurations will yield almost identical heat configurations after a large time $t>0$. Such ill-posedness hinders the recovery of the true heat source location for large $t$. In addition, the forward map $\mathcal{G}$ and its approximation $\mathcal{G}_N$ are defined in terms of an exponential transformation of the eigenvalues of the Laplace-Beltrami operator and the graph Laplacian. Therefore, for moderate $t,$ any small inaccuracy in the eigenvalue estimation can lead to significantly different forward models $\mathcal{G}$ and $\mathcal{G}_N$,  so that $f_N$ is a poor approximation to $f$. If one had access to the true forward map, this issue would not be present.

\section{Discussion}\label{sec: conclusion}
This paper introduced GGP-UCB, a manifold learning technique to optimize an objective function on a hidden   compact   manifold. Our regret bounds and numerical experiments demonstrate the effectiveness of our method. 

\paragraph{Curse of Dimensionality}
Similar to other Bayesian nonparametric techniques, we expect GGP-UCB to be particularly effective when the dimension $m$ of the manifold $\M \subset \R^d$ is small or moderate. In particular, our regret bounds in Theorem \ref{thm:regret bound} suffer from the standard curse of dimension with $m,$ while they do not depend on the dimension $d$ of the ambient space.

\paragraph{Estimating the Intrinsic Dimension}
For simplicity, we have assumed that the dimension $m$ of $\M$ is known, that we have access to clean samples from $\M,$ and that $\mathcal{M}$ has no boundary. If the dimension of $\M$ is unknown, classical manifold learning techniques can be used to estimate it \cite{hein2005intrinsic,harlim2020kernel}. Similarly, if the given point cloud is noisy in that it consists of random perturbations of points sampled from $\M,$ a denoising method can be employed to uncover the underlying geometry \cite{ruiyilocalregularization}. Finally, if $\M$ has a boundary, our GGP-UCB method may be combined with a ghost point diffusion map algorithm to remove boundary artifacts \cite{josh2021,peoples2021spectral,jiang2020ghost}.   

\paragraph{Other Acquisition Functions}
Our focus on UCB acquisition functions was motivated by the desire to establish convergence guarantees under misspecification, as well as by their simplicity and successful empirical performance. However,  there is no algorithmic roadblock to employ other acquisition functions such as expected improvement and Thompson sampling. An interesting direction for future research is to investigate how to  provably correct for geometric misspecification when using these alternative acquisition functions.

\paragraph{Beyond the Manifold Setting}
We have focused on GGP surrogate models defined via a specific choice of unnormalized graph-Laplacian; other graph constructions  (e.g. based on nearest neighbors or self-tuning kernels) and graph-Laplacian (e.g. symmetric and random walk) could be considered \cite{von2007tutorial}. 
Furthermore, the proposed BO framework can be extended beyond the manifold setting. In particular, similar constructions of the GGPs can be carried out over any point cloud (not necessarily embedded in a Euclidean space) as long as a graph Laplacian encoding pairwise similarities of the point cloud can be formed \cite{sanz2020spde,borovitskiy2021matern}. Together with suitable choices of acquisition functions, the resulting framework can be used to solve discrete optimization problems by endowing the search space with a graph structure, which could facilitate the search of optimizers. This is an interesting direction for future research.

\section*{Acknowledgments}
DSA is thankful for the support of NSF DMS-2027056, NSF DMS-2237628, and DOE DE-SC0022232. DSA is also thankful to the BBVA Foundation for a start-up grant. The authors are grateful to Jiaheng Chen for generous feedback on an earlier version of this manuscript.

\bibliographystyle{abbrvnat}
\bibliography{references}

\clearpage
\appendix

These appendices contain the proofs of Proposition \ref{prop:graph GP approx bound}, Theorem \ref{thm:regret bound}, and Corollary \ref{cor:continuum regret}. The proofs build on the theory of spectral convergence of graph Laplacians and regret analysis of Bayesian optimization algorithms. To make our presentation self-contained, we will introduce necessary background and previous results whenever needed.

\section{Proof of Proposition \ref{prop:graph GP approx bound}} \label{appenA}
Let $\M_N$ be i.i.d. samples from a distribution $\mu$ supported on a smooth, connected, and compact $m$-dimensional submanifold $\mathcal{M}\subset \mathbb{R}^d$ without boundary. For simplicity, we shall assume that $\mu$ is the uniform distribution on $\mathcal{M}$. The first result \cite[Theorem 2]{garcia2020error} states that with high probability, the $x_i$'s form a $\rho_N$-net over $\mathcal{M}$ and characterizes $\rho_N$. 
\begin{proposition}\label{prop:d infinity bound}
For any $c>1$, with probability $1-O(N^{-c})$, there exists a transportation map $T_N:\mathcal{M} \to \{x_1, \cdots, x_N\}$ so that 
\begin{align}
    \rho_N:=\underset{x\in\mathcal{M}}{\operatorname{sup}} d_{\mathcal{M}} \bigl(x,T_N(x)\bigr)\lesssim \frac{(\log N)^{p_m}}{N^{1/m}}, \label{eq:rho_N}
\end{align}
where $p_m=3/4$ when $m=2$ and $p_m=1/m$ otherwise. We recall that $d_{\mathcal{M}}$ is the geodesic distance on $\mathcal{M}$. 
\end{proposition}
Proposition \ref{prop:d infinity bound} implies that the point cloud $\M_N$ is ``well-structured'' with high probability and is the building block for the spectral approximation results below \cite[Proposition 10 and Lemma 15]{sanz2020unlabeled}. Recall that the graph-Laplacian $\Delta_N$ constructed in Subsection \ref{ssec:prior} admits a spectral decomposition, in analogy to the Laplace-Beltrami operator $\Delta_{\mathcal{M}}$.

\begin{proposition}\label{prop:eval efun bound}
Suppose there exists $\delta>0$ such that, for $N$ sufficiently large,  
\begin{align}
    h_N\gtrsim N^{-\frac{1}{m+4+\delta}},\quad \quad k_N\lesssim N^{\frac{1-\delta}{m}}, \quad \quad h_Nk_N^{\frac{2}{m}}\lesssim 1. \label{eq:conditions for hn and kn}
\end{align}
Then with probability $1-O(N^{-c})$ for some $c>0$, there exists orthonormalized eigenfunctions $\{\psi_{N,i}\}_{i=1}^N$ for $\Delta_N$,  $\{\psi_i\}_{i=1}^{\infty}$ for $\Delta_{\mathcal{M}}$, and $T_N:\mathcal{M}\rightarrow \{x_1,\ldots,x_N\}$ satisfying $T_N(x_i)=x_i$ so that, for $i=1,\ldots,k_N$,
\begin{align*}
    |\lambda_{N,i}-\lambda_i|&\lesssim \lambda_i\left(\frac{\rho_N}{h_N}+h_N\sqrt{\lambda_i}\right),\\
    \|\psi_{N,i}\circ T_N-\psi_i\|_{\infty}& \lesssim \lambda_i^{m+1}i^{\frac{3}{2}} \sqrt{\frac{\rho_N}{h_N}+h_N\sqrt{\lambda_i}} \,\, .
\end{align*}
\end{proposition}

We also need a result on the growth of the $L^{\infty}$-norm of the Laplace-Beltrami eigenfunctions and their gradients from \cite[Theorem 1.2]{donnelly2006eigenfunctions} and \cite[equation (2.10)]{xu2005asymptotic}. 
\begin{proposition}\label{prop:Laplacian efun infinity bound}
Let $\psi$ be an $L^2$-normalized eigenfunction of $-\Delta_{\mathcal{M}}$ associated with $\lambda\neq 0$. Then $\|\psi\|_{\infty}\leq C\lambda^{(m-1)/4}$ and $\|\nabla \psi\|_{\infty}\leq C\lambda^{(m+1)/2}$ for a universal constant $C$. 
\end{proposition}

\begin{lemma}\label{lemma:SE factor}
The random field $u^{\SE}$ defined in \eqref{eq:continuum SE} satisfies $\mathbb{E}\|u^{\SE}\|^2_{L^2}\asymp 1$ and   has a modification that is locally H\"older continuous of order $\alpha$ for all $\alpha<\frac12$. The random field $u^{\Ma}$ defined in \eqref{eq:continuum Matern} has a modification that is locally H\"older continuous of order $\gamma$ for all $\gamma< \frac{2s-2m+1}{m+3}\wedge \frac12$.  
\end{lemma}
\begin{proof}
By Weyl's law that $\lambda_i\asymp i^{\frac{2}{m}}$ (see e.g. \cite[Theorem 72]{canzani2013analysis}), we have
\begin{align*}
    \mathbb{E}\left\|\sum_{i=1}^{\infty}e^{-\frac{\lambda_i\tau}{2}}\xi_i\psi_i\right\|_{L^2}^2 = \sum_{i=1}^{\infty} e^{-\lambda_i \tau} \asymp \sum_{i=1}^{\infty} e^{-Ci^{2/m}\tau}\asymp \int_0^{\infty}e^{-Cx^{2/m}\tau} \, dx.
\end{align*}
By a change of variable, the last expression is equal to 
\begin{align*}
    \frac{m}{2}\int_0^{\infty} e^{-C\tau x}x^{\frac{m}{2}-1} \, dx =\frac{m}{2}\cdot \frac{\Gamma(\frac{m}{2})}{(C\tau)^{m/2}}\asymp \tau^{-\frac{m}{2}}. 
\end{align*}
For the second claim, by \cite[Corollary 4.5]{lang2016continuity} it suffices to show that 
\begin{align}
    \mathbb{E}|u^{\SE}(x)-u^{\SE}(y)|^2 \leq Cd_{\mathcal{M}}(x,y)^{\eta} \label{eq:cont cond}
\end{align}
for all  $\eta\in(0,1)$, $C>0$, and for all $x,y$ satisfying $d_{\mathcal{M}}(x,y)< 1$. Notice that 
\begin{align*}
    \mathbb{E}|u^{\SE}(x)-u^{\SE}(y)|^2 &=\tau^{\frac{m}{2}}\mathbb{E}\Big|\sum_{i=1}^{\infty}e^{-\frac{\lambda_i\tau}{2}}\xi_i(\psi_i(x)-\psi_i(y))\Big|^2\\
    &=\tau^{\frac{m}{2}} \sum_{i=1}^{\infty} e^{-\lambda_i\tau} |\psi_i(x)-\psi_i(y)|^2 \\
    &\leq \tau^{\frac{m}{2}} \sum_{i=1}^{\infty}e^{-\lambda_i\tau} \|\nabla \psi_i\|_{\infty}^2 d_{\mathcal{M}}(x,y)^2\\
    &\leq  C\tau^{\frac{m}{2}} \sum_{i=1}^{\infty}e^{-\lambda_i\tau} \lambda_i^{m+1} d_{\mathcal{M}}(x,y)^2,
\end{align*}
where we have used Proposition \ref{prop:Laplacian efun infinity bound} in the last step. Now by Weyl's law, 
\begin{align*}
    \sum_{i=1}^{\infty}e^{-\lambda_i \tau}\lambda_i^{m+1} \leq C\sum_{i=1}^{\infty}e^{-c\tau i^{2/m}}i^{\frac{2(m+1)}{m}}\leq C\int_{1}^{\infty}e^{-c\tau x^{2/m}}x^{\frac{2(m+1)}{m}} \, dx<\infty.  
\end{align*}
Therefore, 
\begin{align*}
    \mathbb{E}|u^{\SE}(x)-u^{\SE}(y)|^2 \leq C d_{\mathcal{M}}(x,y)^2 \leq Cd_{\mathcal{M}}(x,y)^{\eta} 
\end{align*}
for any $\eta\in(0,1)$ when $d_{\mathcal{M}}(x,y)<1$, thereby verifying \eqref{eq:cont cond}.

To show local H\"older continuity of $u^{\Ma}$, we need a more careful analysis. Similarly as above, we have 
\begin{align*}
    \mathbb{E}|u^{\Ma}(x)-u^{\Ma}(y)|^2 & = \kappa^{2s-m}\mathbb{E} \Big| \sum_{i=1}^{\infty}(\kappa^2+\lambda_i)^{-\frac{s}{2}}\xi_i (\psi_i(x)-\psi_i(y) \Big|^2\\
    &=\kappa^{2s-m} \sum_{i=1}^{\infty}(\kappa^2+\lambda_i)^{-s} |\psi_i(x)-\psi_i(y)|^2.
\end{align*}
Now by Proposition \ref{prop:Laplacian efun infinity bound}, we shall control $|\psi_i(x)-\psi_(y)|$ by the smaller quantity of the following two bounds  
\begin{align*}
    |\psi_i(x)-\psi_i(y)|&\leq C\lambda_i^{\frac{m-1}{4}},\\
    |\psi_i(x)-\psi_i(y)|&\leq C\lambda_i^{\frac{m+1}{2}}d_{\mathcal{M}}(x,y). 
\end{align*}
Precisely, we have 
\begin{align}
    \mathbb{E}|u^{\Ma}(x)-u^{\Ma}(y)|^2 & \leq C\sum_{i=1}^{\infty}(\kappa^2+\lambda_i)^{-s} \operatorname{min}\left\{\lambda_i^{\frac{m-1}{2}},\lambda_i^{m+1}d_{\mathcal{M}}(x,y)^2\right\}\nonumber\\
    &\leq C\sum_{i=1}^{\infty} i^{-\frac{2s}{m}}\operatorname{min}\left\{i^{\frac{m-1}{m}},i^{\frac{2m+2}{m}}d_{\mathcal{M}}(x,y)^2\right\}\nonumber\\
    &\leq C \sum_{i=1}^{K} i^{-\frac{2s}{m}}i^{\frac{2m+2}{m}} d_{\mathcal{M}}(x,y)^2 + C\sum_{i=K+1}^{\infty} i^{-\frac{2s}{m}}i^{\frac{m-1}{m}}, \label{eq:ma holder}
\end{align}
where $K=d_{\mathcal{M}}(x,y)^{-\frac{2m}{m+3}}$. Therefore we have 
\begin{align*}
    \eqref{eq:ma holder} &\leq Cd_{\mathcal{M}}(x,y)^2 \int_1^K z^{\frac{-2s+2m+2}{m}}dz+C\int_{K}^{\infty}z^{\frac{-2s+m-1}{m}}dz\\
    &\leq Cd_{\mathcal{M}}(x,y)^{\frac{4s-4m+2}{m+3}} \leq Cd_{\mathcal{M}}(x,y). 
\end{align*}
The result follows again by \cite[Corollary 4.5]{lang2016continuity}. 
  
\end{proof}

Now we are ready to prove Proposition \ref{prop:graph GP approx bound}. The first statement on the approximation error of $u_N^{\Ma}$ follows from \cite[Theorem 4.6]{sanz2020unlabeled}. To show the second, recall that 
\begin{alignat*}{2}
    u_N^{\SE}&=\tau^{\frac{m}{4}}\sum_{i=1}^{k_N}e^{-\frac{\lambda_{N,i} \tau}{2}}\xi_i\psi_{N,i},\quad \quad &&\xi_i\overset{i.i.d.}{\sim}\mathcal{N}(0,1),\\
    u^{\SE}&=\tau^{\frac{m}{4}}\sum_{i=1}^{\infty} e^{-\frac{\lambda_i \tau}{2}}\xi_i\psi_{i},\quad \quad &&\xi_i\overset{i.i.d.}{\sim}\mathcal{N}(0,1),
\end{alignat*}
and introduce two intermediate random processes
\begin{alignat*}{2}
    \widetilde{u}_{N}^{\SE} &= \tau^{\frac{m}{4}}\sum_{i=1}^{k_N} e^{-\frac{\lambda_i \tau}{2}}\xi_i\psi_{N,i},\quad \quad &&\xi_i\overset{i.i.d.}{\sim}\mathcal{N}(0,1), \\
    \widehat{u}_N^{\SE}&=\tau^{\frac{m}{4}}\sum_{i=1}^{k_N} e^{-\frac{\lambda_i \tau}{2}}\xi_i\psi_{i},\quad \quad &&\xi_i\overset{i.i.d.}{\sim}\mathcal{N}(0,1).
\end{alignat*}
We then have
\begin{align*}
    \mathbb{E}\|u_N^{\SE}\circ T_N-u^{\SE}\|_{\infty}
    \leq \mathbb{E}\|u_N^{\SE}\circ T_N-\widetilde{u}_{N}^{\SE}\circ T_N\|_{\infty}+\mathbb{E}\|\widetilde{u}_{N}^{\SE}\circ T_N-\widehat{u}_N^{\SE}\|_{\infty}+\mathbb{E}\|\widehat{u}_N^{\SE}-u^{\SE}\|_{\infty}
\end{align*}
and we shall proceed by bounding each of the three terms on the right. First, note that 
\begin{align}
    \mathbb{E}\|\widehat{u}_N^{\SE}-u^{\SE}\|_{\infty}
    &=\mathbb{E}\bigg\|\tau^{\frac{m}{4}}\sum_{i=k_N+1}^{\infty}e^{-\frac{\lambda_i \tau}{2}}\xi_i\psi_i\bigg\|_{\infty}\nonumber\\
    &\lesssim \sum_{i=k_N+1}^{\infty}e^{-\frac{\lambda_i \tau}{2}}\mathbb{E}|\xi_i|\|\psi_i\|_{\infty}
    \lesssim  \sum_{i=k_N+1}^{\infty}e^{-\frac{\lambda_i \tau}{2}} \lambda_i^{\frac{m-1}{4}},\label{eq:SE eq1}
\end{align}
where we have used Proposition \ref{prop:Laplacian efun infinity bound} in the last step. Now by Weyl's law, $\lambda_i\asymp i^{2/m}$ so that we can further bound 
\begin{align}
    \eqref{eq:SE eq1} \lesssim \sum_{i=k_N+1}^{\infty} e^{-c_0  \tau i^{2/m}}i^{\frac{2}{m}\frac{m-1}{4}}&\lesssim \int_{k_N}^{\infty} e^{-c_0\tau x^{2/m}} x^{\frac{m-1}{2m}} \, dx \nonumber\\
    &= \int_{k_N^{2/m}}^{\infty}e^{-c_0 \tau z}z^{\frac{3m-1}{4}-1}dz \label{eq:SE eq2}
\end{align}
after a change of variable, where $c_0$ is a universal constant. Notice that the rightmost term \eqref{eq:SE eq2} is equal up to a multiplicative constant to $\mathbb{P}(X\geq k_N^{2/m})$ with $X$ being a Gamma random variable with shape parameter $\frac{3m-1}{4}$ and scale parameter $\frac{1}{c_0 \tau}$. Now by the tail bound of sum-Gamma distributions (cf. \cite[Lemma 5.1]{zhang2020concentration}) applied to $X-\mathbb{E}X\in \text{sub}\Gamma \Bigl(\frac{3m-1}{4c_0^2 \tau^2},\frac{1}{c_0 \tau}\Bigr)$, we have 
\begin{align}
     \mathbb{E}\|\widehat{u}_N^{\SE}-u^{\SE}\|_{\infty}\lesssim\eqref{eq:SE eq2} \lesssim \mathbb{P}(X-\mathbb{E}X\geq k_N^{2/m}-\mathbb{E}X)\lesssim e^{-Ck_N^{2/m}} \label{eq:SE part1}
\end{align}
for some constant $C$ when $k_N^{2/m} \gg \mathbb{E}X=\frac{3m-1}{4c_0 \tau}$. Similarly, we have 
\begin{align}
    \mathbb{E}\|u_N^{\SE}\circ T_N-\widetilde{u}_{N}^{\SE}\circ T_N\|_{\infty}
    \lesssim \sum_{i=1}^{k_N} \Big|e^{-\frac{\lambda_{N,i}\tau}{2}}-e^{-\frac{\lambda_i \tau}{2}}\Big| \|\psi_{N,i}\circ T_N\|_{\infty}. \label{eq:SE eq3}
\end{align}
By the mean value theorem, we have that $|e^{-x}-e^{-y}|=e^{-\zeta}|x-y|\leq \text{max}\{e^{-x},e^{-y}\}|x-y|$ for some $\zeta \in \bigl(\min(x,y), \max(x,y)\bigr)$ where $x,y>0$. Thus, we have
\begin{align*}
    \Big|e^{-\frac{\lambda_{N,i}\tau}{2}}-e^{-\frac{\lambda_i \tau}{2}}\Big| &\leq \text{max} \Bigl\{e^{-\frac{\lambda_{N,i}\tau}{2}},e^{-\frac{\lambda_i \tau}{2}} \Bigr\} \frac{\tau}{2}|\lambda_{N,i}-\lambda_i| 
    \leq \frac{\tau}{2} e^{-\frac{\lambda_i\tau}{4}}\lambda_i \left(\frac{\rho_N}{h_N}+h_N\sqrt{\lambda_i}\right),
\end{align*}
where in the last step we have used Proposition \ref{prop:eval efun bound} which also implies $\lambda_{N,i}\geq \lambda_i/2$ when $N$ is large. Moreover, Proposition \ref{prop:eval efun bound} implies that, for $i=1,\ldots,k_N$,
\begin{align}
    \|\psi_{N,i}\circ T_N\|_{\infty} &\leq \|\psi_{N,i}\circ T_N-\psi_i\|_{\infty}+\|\psi_i\|_{\infty}\nonumber \\
    &\lesssim \lambda_i^{m+1}i^{\frac{3}{2}} \sqrt{\frac{\rho_N}{h_N}+h_N\sqrt{\lambda_i}} +\|\psi_i\|_{\infty} 
    \lesssim \Big(\frac{\rho_N}{h_N}+h_N\Big)i^{\frac{7m+5}{2m}}+\|\psi_i\|_{\infty}. \label{eq:SE eq4}
\end{align}
Proposition \ref{prop:Laplacian efun infinity bound} implies that $\|\psi_i\|_{\infty}\lesssim \lambda_i^{\frac{m-1}{4}}\lesssim i^{\frac{m-1}{2m}}$. Therefore we would like to set $h_N$ and $k_N$ to satisfy 
\begin{align}
    \Big(\frac{\rho_N}{h_N}+h_N\Big)k_N^{\frac{7m+5}{2m}} \lesssim k_N^{\frac{m-1}{2m}} \label{eq:condition4 for kn}
\end{align}
so that \eqref{eq:SE eq4} grows like $\|\psi_i\|_{\infty}$ for all $i=1,\ldots,k_N$. We shall keep \eqref{eq:condition4 for kn} in mind together with those conditions in \eqref{eq:conditions for hn and kn} and proceed by assuming that we have made such choices. 
Now we can bound 
\begin{align}
    \eqref{eq:SE eq3} &\lesssim \sum_{i=1}^{k_N} \frac{\tau}{2}e^{-\frac{\lambda_i \tau}{4}}\lambda_i^{\frac{m+3}{4}}\left(\frac{\rho_N}{h_N}+h_N\sqrt{\lambda_i}\right)\nonumber\\
    &\lesssim \frac{\tau}{2}\left(\frac{\rho_N}{h_N}+h_N\right)\sum_{i=1}^{k_N}e^{-\frac{\lambda_i\tau}{4}}\lambda_i^{\frac{m+5}{4}}\lesssim \frac{\rho_N}{h_N}+h_N, \label{eq:SE part2}
\end{align}
where we used the fact that
\begin{align*}
  \sum_{i=1}^{k_N}e^{-\frac{\lambda_i \tau}{4}}\lambda_i^{\frac{m+5}{4}} \lesssim \sum_{i=1}^{k_N}e^{-Ci^{2/m}}i^{\frac{m+5}{2m}}\lesssim \int_{1}^{\infty}e^{-Cx^{2/m}}x^{\frac{m+5}{2m}}\, dx <\infty. 
\end{align*}
Lastly, we have by Proposition \ref{prop:eval efun bound}
\begin{align}
    \mathbb{E}\|\widetilde{u}_{N}^{\SE}\circ T_N-\widehat{u}_N^{\SE}\|_{\infty}& \lesssim \sum_{i=1}^{k_N}e^{-\frac{\lambda_i \tau}{2}} \|\psi_{N,i}\circ T_N-\psi_i\|_{\infty}\nonumber\\
    &\lesssim \sum_{i=1}^{k_N}e^{-\frac{\lambda_i \tau}{2}}\lambda_i^{m+1}i^{\frac{3}{2}}\sqrt{\frac{\rho_N}{h_N}+h_N\sqrt{\lambda_i}}\lesssim \sqrt{\frac{\rho_N}{h_N}+h_N} \,. \label{eq:SE part3}
\end{align}
Combining \eqref{eq:SE part1}, \eqref{eq:SE part2},\eqref{eq:SE part3}, we get 
\begin{align*}
    \mathbb{E}\|u_N^{\SE}\circ T_N-u^{\SE}\|_{\infty} \lesssim e^{-Ck_N^{2/m}} + \sqrt{\frac{\rho_N}{h_N}+h_N} \,.
\end{align*}
Now it remains to set $h_N$ and $k_N$ and we remark that the approximation error will be dominated by the second term $\sqrt{\rho_N/h_N+h_N}$ when $N$ is large. It can be checked that the following scaling satisfies the conditions imposed by \eqref{eq:conditions for hn and kn} and \eqref{eq:condition4 for kn}. 
\paragraph{Case 1: $m\leq 4$} Setting for some arbitrarily small $\delta>0$ 
\begin{align*}
    h_N\asymp N^{-\frac{1}{m+4+\delta}}(\log N)^{\frac{p_m}{2}} ,\quad \quad (\log N)^{\frac{m}{2}}\ll k_N\ll N^{\frac{m}{(m+4+\delta)(3m+3)}}(\log N)^{-\frac{mp_m}{6m+6}},
\end{align*}
we obtain that, for large $N,$
\begin{align*}
    \mathbb{E}\|u_N^{\SE}\circ T_N-u^{\SE}\|_{\infty}\lesssim N^{-\frac{1}{2(m+4+\delta)}} (\log N)^{\frac{p_m}{4}}.
\end{align*}
\paragraph{Case 2: $m\geq 5$} Setting
\begin{align*}
    h_N\asymp N^{-\frac{1}{2m}}(\log N)^{-\frac{p_m}{2}},\quad \quad (\log N)^{\frac{m}{2}}\ll k_N\ll N^{\frac{1}{6m+6}}(\log N)^{-\frac{mp_m}{6m+6}},
\end{align*}
we obtain 
\begin{align*}
    \mathbb{E}\|u_N^{\SE}\circ T_N-u^{\SE}\|_{\infty}\lesssim N^{-\frac{1}{4m}}(\log N)^{\frac{p_m}{4}}.
\end{align*}

\section{Proof of Theorem \ref{thm:regret bound}}  \label{appenB}
We start by introducing the key ingredients of the regret analysis of Bayesian optimization algorithms, in particular the GGP-UCB algorithm. Most of the preliminary results in this section can be found in \cite{srinivas2010gaussian,bogunovic2021misspecified}. 

Recall that our goal is to bound the simple regret defined as in \eqref{eq:simple regret}. But a typical strategy in the BO literature is to look at the \emph{cumulative regret},  defined as 
\begin{align}
    R_{N,L}=\sum_{\ell=1}^L \Bigl( f(z^{\ast})-f(z_\ell)\Bigr),\quad \quad    z^*=\underset{z\in \M_N}{\operatorname{arg\,max}}\, f(z).  \label{eq:cumu R_{N,L}}
\end{align}
Then using the fact that 
\begin{align*}
    f(z_L^*)\geq \frac{1}{L}\sum_{\ell=1}^L f(z_\ell), \quad \quad z_L^*= \underset{z\in\{z_\ell\}_{\ell=1}^L}{\operatorname{arg\,max}}\, f(z),
\end{align*}
one can bound the simple regret as 
\begin{align}
    r_{N,L}=f(z^*)-f(z_L^*) \leq \frac{1}{L}\sum_{\ell=1}^L \Bigl(f(z^*)-f(z_\ell)\Bigr) =\frac{R_{N,L}}{L}.  \label{eq:sr cr}
\end{align}
The key to bounding the cumulative regret consists of two steps. The first is a concentration-type result that constructs confidence bands which $f$ lies in with high probability based on the observed samples. More precisely, we have the following result. 
\begin{lemma}\label{lem:graph GP concentration}
Let $\delta\in(0,1)$ and set $b_{N,\ell}=\sqrt{2\log (\pi^2\ell^2 N/6\delta)}$. Then with probability $1-\delta$, we have
\begin{align*}
    |u_N(z)-\widetilde{\mu}_{N,\ell-1}(z)|\leq b_{N,\ell}\sigma_{N,\ell-1}(z) \quad \forall z \in \M_N, \quad \forall \ell\geq 1,
\end{align*}
where 
\begin{align*}
     \widetilde{\mu}_{N,\ell}( z)&=c_{N,\ell}(z)^\top(C_{N,\ell}+\sigma^2 I)^{-1}\widetilde{Y}_{N,\ell}  \quad \quad 
\end{align*}
and $\widetilde{Y}_{N,\ell}\in \mathbb{R}^\ell$ is vector with entry $(\widetilde{Y}_{N,\ell})_i=u_N(z_i)+\eta_i.$ See \eqref{eq:pm pstd} for the definition of $c_{N,\ell}$ and $C_{N,\ell}$.
\end{lemma}
\begin{proof}
This is \cite[Lemma 5.1]{srinivas2010gaussian} applied to the graph GP $u_N$, with the ``surrogate'' data $\widetilde{Y}_{N,\ell}$. 
\end{proof}
Here and below, we shall use $c_N(\cdot,\cdot)$ as a placeholder for either the Mat\'ern or SE graph-based covariance function \eqref{eq:graph cf}. Notice that the ``surrogate'' data $\widetilde{Y}_{N,\ell}$ is introduced only for the purpose of analysis and the algorithm only has access to the real data $y_\ell=f(z_\ell)+\eta_\ell$. An important follow-up question is on the difference between the surrogate-data posterior mean $\widetilde{\mu}_{N,\ell}$ and the true posterior mean $\mu_{N,\ell}=c_{N,\ell}(x)^\top(C_{N,\ell}+\sigma^2 I)^{-1}Y_\ell$ that is actually used in the algorithm, answered by the following result. 

\begin{lemma}\label{lem:misspec pm bound}
In the event of \eqref{eq:misspec bound}, we have
\begin{align*}
    |\mu_{N,\ell}(z)-\widetilde{\mu}_{N,\ell}(z)|\leq \frac{\epsilon_N\sqrt{\ell}}{\delta \sigma} \sigma_{N,\ell}(z), \quad \forall z\in \M_N\, \quad \forall \ell\geq 1,  
\end{align*}
where we recall $\sigma$ is the standard deviation of the noise $\eta_{\ell}$ and $\sigma_{N,\ell}$ is defined in \eqref{eq:pm pstd}.
\end{lemma}
\begin{proof}
This follows by setting the misspecification error to be $\epsilon_N/\delta$ in \cite[Lemma 2]{bogunovic2021misspecified}.
\end{proof}

Now with these preparations, we are ready to start the proof of Theorem \ref{thm:regret bound}. 
In the event of \eqref{eq:misspec bound} that 
\begin{align*}
    \underset{z\in\M_N}{\operatorname{max}} |u_N(z)-f(z)|\leq \delta^{-1}\epsilon_N,
\end{align*}
which holds with probability $1-\delta$ by Proposition \ref{prop:graph GP approx bound} (with $\epsilon_N$ the corresponding error bounds \eqref{eq:epsilon_N}), we can shift our focus to the following cumulative regret 
\begin{align*}
    \widetilde{R}_{N,L}=\sum_{\ell=1}^L  u_N(z^*)- u_N(z_\ell), \quad\quad z^*=\underset{z\in \M_N}{\operatorname{arg\,max}}\, f(z),
\end{align*}
which differs from $R_{N,L}$ \eqref{eq:cumu R_{N,L}} at most by $2\epsilon_N L/\delta$. Under the further event where Lemma \ref{lem:graph GP concentration} holds, we have by Lemma \ref{lem:misspec pm bound} that for all $z\in \M_N$,
\begin{align*}
    \mu_{N,\ell-1}(z)-\left(b_{N,\ell}+\frac{\epsilon_N\sqrt{\ell-1}}{\delta \sigma} \right)\sigma_{N,\ell-1}(z)\leq u_N(z)
    \leq \mu_{N,\ell-1}(z)+\left(b_{N,\ell}+\frac{\epsilon_N\sqrt{\ell-1}}{\delta \sigma} \right)\sigma_{N,\ell-1}(z).
\end{align*}
Therefore 
\begin{align*}
    \widetilde{R}_{N,L}&\leq \sum_{\ell=1}^L \Biggl(\mu_{N,\ell-1}(z^*)+\left(b_{N,\ell}+\frac{\epsilon_N\sqrt{\ell-1}}{\delta\sigma} \right)\sigma_{N,\ell-1}(z^*)\\
    &\quad \quad -\left[\mu_{N,\ell-1}(z_\ell)-\left(b_{N,\ell}+\frac{\epsilon_N\sqrt{\ell-1}}{\delta \sigma} \right)\sigma_{N,\ell-1}(z_\ell)\right] \Biggr)\\
    &\leq 2\sum_{\ell=1}^L\left(b_{N,\ell}+\frac{\epsilon_N\sqrt{\ell-1}}{\delta\sigma} \right)\sigma_{N,\ell-1}(z_\ell)\\
    &\leq 2\left(b_{N,L}+\frac{\epsilon_N\sqrt{L-1}}{\delta \sigma} \right)\sum_{\ell=1}^L \sigma_{N,\ell-1}(z_\ell),
\end{align*}
where in the second step we have used our definition of $z_\ell$ in \eqref{eq:choice of zt} that for all $z\in \M_N$ including $z^*$
\begin{align*}
    \mu_{N,\ell-1}(z_\ell)+\left(b_{N,\ell}+\frac{\epsilon_N\sqrt{\ell-1}}{\delta \sigma} \right)\sigma_{N,\ell-1}(z_\ell)\geq 
    \mu_{N,\ell-1}(z)+\left(b_{N,\ell}+\frac{\epsilon_N\sqrt{\ell-1}}{\delta \sigma} \right)\sigma_{N,\ell-1}(z).
\end{align*}
Therefore we have arrived at the conclusion that 
\begin{align}
    R_{N,L}\leq \frac{2\epsilon_NL}{\delta} + 2\left(b_{N,L}+\frac{\epsilon_N\sqrt{L-1}}{\delta\sigma} \right)\sum_{\ell=1}^L \sigma_{N,\ell-1}(z_\ell). \label{eq:proof RNT eq1}
\end{align}
Here comes the second key ingredient in the regret analysis, which is to relate the sum of posterior standard deviations $\sum_{\ell=1}^L\sigma_{N,\ell-1}(z_\ell)$ to the so-called \emph{maximum information gain}.
The following result is taken from \cite[Lemma 5.3]{srinivas2010gaussian}. 

\begin{lemma}\label{lemma:information gain}
Let $I(y;v)$ denote the mutual information between two random vectors $y$ and $v$ of the same size. We have 
\begin{align*}
    I\big(\widetilde{Y}_{N,L}; \{u_N(z_\ell)\}_{\ell=1}^L\big)=\frac12\sum_{\ell=1}^L \log \big(1+\sigma^{-2} \sigma^2_{N,\ell-1}(z_\ell)\big),
\end{align*}
where $\widetilde{Y}_{N,L}$ is the surrogate data defined in Lemma \ref{lem:graph GP concentration}. 
\end{lemma}

As a corollary, we have the following result. 
\begin{lemma}\label{lem:sum of pstd bounded by information gain}
For $N$ large, there exists a universal constant $B$ such that $c_N(x,\tilde{x})$ $\leq B$.   
Moreover, 
\begin{align*}
    \sum_{\ell=1}^L \sigma_{N,\ell-1}(z_\ell)\leq \sqrt{2(\sigma^2+B^2)L\gamma_L}\,,
\end{align*}
where 
\begin{align*}
    \gamma_L= \underset{S\subset \M_N, |S|=L}{\operatorname{max}}\, I\big(\widetilde{Y}_{N,S}; u_N(S)\big)
\end{align*}
is the maximum information gain. Here $u_N(S)$ denotes the vector $\{u_N(s)\}_{s\in S}$ and $\widetilde{Y}_{N,S}$ is the associated vector of observations as in Lemma \ref{lem:graph GP concentration}. 
\end{lemma}
\begin{proof}
The first statement can be proved in a similar fashion as Proposition \ref{prop:graph GP approx bound} by bounding the difference $|c_N(z_1,z_2)-c(z_1,z_2)|$ between the graph and manifold covariance functions, and using the fact that the manifold covariance function $c(\cdot,\cdot)$ is uniformly upper bounded (which follows by the control of growth of the Laplace-Beltrami eigenfunctions in Proposition \ref{prop:Laplacian efun infinity bound}). 

For the second statement, notice that $\sigma_{N,\ell-1}(z_\ell)\leq c_N(z_\ell,z_\ell)\leq B$. Using the fact that $(1+\sigma^{-2}B^2)\log (1+ x)\geq x$ over $[0,\sigma^{-2}B^2]$, we have \begin{align*}
    \sum_{\ell=1}^L \sigma^2_{N,\ell-1}(z_\ell)&\leq  (\sigma^2+B^2)  \sum_{\ell=1}^L \log \big(1+\sigma^{-2} \sigma^2_{N,\ell-1}(z_\ell)\big)\\
    &=2(\sigma^2+B^2)  I(\{y_\ell\}_{\ell=1}^L; \{f_{N}(z_\ell)\}_{\ell=1}^L)\leq 2( \sigma^2+B^2) \gamma_L,
\end{align*}
where the equality in the second step follows from Lemma \ref{lemma:information gain}. 
Finally, by Cauchy-Schwarz inequality we have that  $\sum_{\ell=1}^L\sigma_{N,\ell-1}(z_\ell)\leq \sqrt{L\sum_{\ell=1}^L \sigma_{N,\ell-1}(z_\ell)^2}$ and the result follows. 
\end{proof}

Applying Lemma \ref{lem:sum of pstd bounded by information gain} to \eqref{eq:proof RNT eq1}, we get 
\begin{align}
    R_{N,L}\leq C \left(b_{N,L}\sqrt{L}+\frac{\epsilon_NL}{\delta\sigma}\right) \sqrt{\gamma_L}\label{eq:proof RNT eq2},
\end{align}
where $C$ is a universal constant. Upper bounds on $\gamma_L$ have been studied extensively in the literature and by \cite[Theorem 3 or equation (7)]{vakili2021information} with $D=k_N$ and $\delta_D=0$ in our case (which holds because our graph kernel only has $k_N$ nonzero eigenvalues), we get
\begin{align*}
    R_{N,L}\leq C \left(b_{N,L}\sqrt{L}+\frac{\epsilon_NL}{\delta\sigma}\right) \sqrt{k_N\log L}.
\end{align*}
Finally, we return to bounding the simple regret using \eqref{eq:sr cr}:
\begin{align*}
    r_{N,L}\leq \frac{R_{N,L}}{L} \leq C \left(\frac{b_{N,L}}{\sqrt{L}}+\frac{\epsilon_N}{\delta\sigma}\right) \sqrt{k_N\log L}. 
\end{align*}

\section{Proof of Corollary \ref{cor:continuum regret}}    
  
Denote $\hat{z}_N^*=\operatorname{arg\,min}_{z\in\mathcal{M}_N} \, d_{\mathcal{M}}(z^*,z)$, i.e., the point in $\mathcal{M}_N$ closest to $z^*$. Then by Proposition \ref{prop:d infinity bound} we necessarily have 
\begin{align*}
    d_{\mathcal{M}}(z^*,\hat{z}_N^*) \leq d_{\mathcal{M}}(z^*,T_N(z^*))\leq \rho_N. 
\end{align*}
Now notice that 
\begin{align*}
    f(z^*)-f(z_N^*)= [f(z^*)-f(\hat{z}_N^*)] +[f(\hat{z}_N^*) -f(z_N^*)]\leq f(z^*)-f(\hat{z}_N^*) 
\end{align*}
since $z_N^*$ being the maximizer of $f$ over $\mathcal{M}_N$ implies $f(\hat{z}_N^*) -f(z_N^*)\leq 0$. By local $\alpha$-H\"older continuity of $f$ at $z^*$, we conclude that 
\begin{align*}
    f(z^*)-f(\hat{z}_N^*) \leq C_f d_{\mathcal{M}}(z^*,\hat{z}_N^*)^{\alpha} \leq C\rho_N^{\alpha}.
\end{align*}
By Lemma \ref{lemma:SE factor} and \eqref{eq:rho_N}, we get 
\begin{align}\label{eq:holder error}
    f(z^*)-f(z_N^*)=\widetilde{O} 
    \begin{cases}
    N^{-\left[\frac{2s-2m+1}{m(m+3)}\wedge \frac{1}{2m}\right]}\qquad & \text{(Mat\'ern)}\\
    N^{-\frac{1}{2m}}  \qquad & \text{(SE)}
    \end{cases},
\end{align}
where we have dropped all dependence on logarithmic factors in the notation $\widetilde{O}$. 
The results follows by the identity
\begin{align*}
    f(z^*)-f(z_L^*)= f(z^*)-f(z_N^*) +r_{N,L}
\end{align*}
and the observation that the error in \eqref{eq:holder error} would be absorbed by that of $r_{N,L}$ as shown in \eqref{eq:regret explicit rate}.

\end{document}